\documentclass[11pt]{article}

\usepackage{amsmath, natbib, amssymb}
\usepackage{multirow, graphicx, amsthm}
\usepackage{url}
\usepackage[bookmarks=false]{hyperref}

\def\be{{\mathbf e}}
\def\bd{{\mathbf d}}
\def\bk{{\mathbf k}}

\def\bw{{\mathbf w}}
\def\bx{{\mathbf x}}
\def\by{{\mathbf y}}
\def\bz{{\mathbf z}}
\def\bX{{\mathbf X}}
\def\bZ{{\mathbf Z}}
\def\bzero{{\mathbf 0}}
\def\bones{{\mathbf 1}}

\newcommand{\balpha}{\mbox{\boldmath$\alpha$}}

\newcommand{\reals}{{\mathbb{R}}}
\DeclareMathOperator*{\argmin}{arg\,min}

\newtheorem{theorem}{Theorem}

\newtheorem{lemma}{Lemma}

\newtheorem{corollary}{Corollary}
\newtheorem{definition}{Definition}
\pdfoutput=1

\begin{document}

\title{Robust Kernel Density Estimation}

\author{JooSeuk Kim$^1$ and Clayton D. Scott$^{1,2}$\\
$^1$Electrical Engineering and Computer Science, $^2$Statistics\\
University of Michigan, Ann Arbor, MI 48109-2122 USA\\
email: \texttt{\{stannum, clayscot\}@umich.edu} }


\maketitle

\begin{abstract}
We propose a method for nonparametric density estimation that exhibits robustness to contamination of the training sample.
This method achieves robustness by combining a traditional kernel density estimator (KDE) with ideas from classical $M$-estimation.
We interpret the KDE based on a radial, positive semi-definite kernel as a sample mean in the associated reproducing kernel Hilbert space. Since the sample mean is sensitive to outliers, we estimate it robustly via $M$-estimation, yielding a robust kernel density
estimator (RKDE).

An RKDE can be computed efficiently via a kernelized iteratively re-weighted least squares (IRWLS) algorithm. Necessary and sufficient conditions are given for kernelized IRWLS to converge to the global minimizer of the $M$-estimator objective function.
The robustness of the RKDE is demonstrated with a representer theorem, the influence function, and experimental results for density estimation and anomaly detection.
\end{abstract}

\section{Introduction}\label{sec:intro}
The kernel density estimator (KDE) is a well-known nonparametric estimator
of univariate or multivariate densities, and numerous articles have been
written on its properties, applications, and extensions
\citep{silve86,scottdw92}. However, relatively little work has been done
to
understand or improve the KDE in situations where the training sample is
contaminated. This paper addresses a method of nonparametric density
estimation that generalizes the KDE, and exhibits robustness to
contamination of the training sample.
\let\thefootnote\relax\footnote{Shorter versions of this work previously
appeared at the International Conference on Acoustics, Speech, and Signal
Processing \citep{jkim08} and the International Conference on Machine
Learning \citep{kim11}.}

Consider training data following a contamination model
\begin{equation}
\label{eqn:contam}
\bX_1, \dots, \bX_n \stackrel{iid}{\sim} (1-p) f_0 + p f_1,
\end{equation}
where $f_0$ is the ``nominal" density to be estimated, $f_1$ is the density of the contaminating distribution, and $p < \frac12$ is the proportion of contamination. Labels are not available, so that the problem is unsupervised. The objective is to estimate $f_0$ while making no parametric assumptions about the nominal or contaminating distributions.

Clearly $f_0$ cannot be recovered if there are {\em no} assumptions on $f_0, f_1$ and $p$.  Instead, we will focus on a set of nonparametric conditions that are reasonable in many practical applications.  In particular, we will assume that, relative to the nominal data, the contaminated data are
\begin{description}
\item[(a)] {\em outlying}: the densities $f_0$ and $f_1$ have relatively
little overlap
\item[(b)] {\em diffuse}: $f_1$ is not too spatially concentrated relative
to $f_0$
\item[(c)] {\em not abundant}: a minority of the data come from $f_1$
\end{description}
Although we will not be stating these conditions more precisely, they capture the intuition behind the quantitative results presented below.

As a motivating application, consider anomaly detection in a computer
network. Imagine that several multi-dimensional measurements $\bX_1,
\dots, \bX_n$ are collected.  For example, each $\bX_i$ may record the
volume of traffic along certain links in the network, at a certain instant
in time \citep{chhabra08}. If each measurement is collected when the
network is in a nominal state, these data could be used to construct an
anomaly detector by first estimating the density $f_0$ of nominal
measurements, and then thresholding that estimate at some level to obtain
decision regions. Unfortunately, it is often difficult to know that the
data are free of anomalies, because assigning labels (nominal vs.
anomalous) can be a tedious, labor intensive task. Hence, it is necessary
to estimate the nominal density (or a level set thereof) from contaminated
data. Furthermore, the distributions of both nominal and anomalous
measurements are potentially complex, and it is therefore desirable to
avoid parametric models.

The proposed method achieves robustness by combining a traditional kernel
density estimator with ideas from $M$-estimation \citep{huber64,hampel74}.
The KDE based on a radial, positive semi-definite (PSD) kernel is
interpreted as a sample mean in the reproducing kernel Hilbert space
(RKHS) associated with the kernel. Since the sample mean is sensitive to
outliers, we estimate it robustly via $M$-estimation, yielding a robust
kernel density estimator (RKDE). We describe a kernelized iteratively
re-weighted least squares (KIRWLS) algorithm to efficiently compute the
RKDE, and provide necessary and sufficient conditions for the convergence
of KIRWLS to the RKDE.

We also offer three arguments to support the claim that the RKDE robustly
estimates the nominal density and its level sets. First, we characterize
the RKDE by a representer theorem. This theorem shows that the RKDE is a
weighted KDE, and the weights are smaller for more outlying data points.
Second, we study the influence function of the RKDE, and show through an
exact formula and numerical results that the RKDE is less sensitive to
contamination by outliers than the KDE. Third, we conduct experiments on
several benchmark datasets that demonstrate the improved performance of
the RKDE, relative to competing methods, at both density estimation and
anomaly detection.

One motivation for this work is that the traditional kernel density
estimator is well-known to be sensitive to outliers. Even without
contamination, the standard KDE tends to overestimate the density in
regions where the true density is low.  This has motivated several authors
to consider variable kernel density estimators (VKDEs), which employ a
data-dependent bandwidth at each data point \citep{breiman77, abramson82,
terrell92}.  This bandwidth is adapted to be larger where the data are
less dense, with the aim of decreasing the aforementioned bias. Such methods have
been applied in outlier detection and computer vision
applications \citep{meer01, latecki07}, and are one
possible approach to robust nonparametric density estimation. We compare
against these methods in our experimental study.

Density estimation with positive semi-definite kernels has been studied by
several authors. \citet{vapnik00supportvector} optimize a criterion based
on the empirical cumulative distribution function over the class of
weighted KDEs based on a PSD kernel. \citet{shawetaylor07} provide a
refined theoretical treatment of this approach. \citet{gretton08moment}
adopt a different criterion based on Hilbert space embeddings of
probability distributions. Our approach is somewhat similar in that we
attempt to match the mean of the empirical distribution in the RKHS, but
our criterion is different. These methods were also not designed with contaminated data in mind.

We show that the standard kernel density estimator can be viewed as the
solution to a certain least squares problem in the RKHS.  The use of
quadratic criteria in density estimation has also been previously
developed.  The aforementioned work of Song et al. optimizes the
norm-squared in Hilbert space, whereas \citet{kimthesis, girolami03,
kim10pami, gray11cake} adopt the integrated squared error. Once again, these methods are not designed for contaminated data.


Previous work combining robust estimation and kernel methods has focused
primarily on supervised learning problems. $M$-estimation applied to
kernel regression has been studied by various authors
\citep{christmann07,debruyne08a,debruyne08b,zhu08,wibowo,brabanter}.
Robust surrogate losses for kernel-based classifiers have also been
studied \citep{schurmanns}. In unsupervised learning, a robust way of
doing kernel principal component analysis, called spherical KPCA, has been
proposed, which applies PCA to feature vectors projected onto a unit
sphere around the spatial median in a kernel feature space
\citep{hubert10}. The kernelized spatial depth was also proposed
to estimate depth contours nonparametrically \citep{yixin09}. To
our knowledge, the RKDE is the first application of $M$-estimation ideas
in kernel density estimation.


In Section \ref{sec:rkde} we propose robust kernel density estimation. In
Section \ref{sec:representer} we present a representer theorem for the
RKDE. In Section \ref{sec:kirwls} we describe the KIRWLS algorithm and its
convergence. The influence function is developed in Section \ref{sec:ic},
and experimental results are reported in Section \ref{sec:experiment}.
Conclusions are offered in Section \ref{sec:conclusion}. Section
\ref{sec:proofs} contains proofs of theorems. Matlab code implementing our
algorithm is available at \url{www.eecs.umich.edu/~cscott}.

\section{Robust Kernel Density Estimation}\label{sec:rkde}
Let $\bX_1, \dots, \bX_n \in \reals^d$ be a random sample from a distribution $F$ with a density $f$. The
kernel density estimate of $f$, also called the Parzen window estimate, is a nonparametric estimate given by
\begin{equation*}
\widehat{f}_{KDE}\left(\bx\right) = \frac{1}{n} \sum_{i = 1}^{n} k_\sigma\left(\bx , \bX_i\right)
\end{equation*}
where $k_\sigma$ is a kernel function with bandwidth $\sigma$. To ensure that $\widehat{f}_{KDE}(\bx)$ is a density, we assume the kernel function satisfies
$k_\sigma(\,\cdot\,, \,\cdot\,) \geq 0$ and $\int k_\sigma\left(\bx, \,\cdot\,\right) \, d\bx = 1$. We will also assume that $k_\sigma(\bx, \bx')$ is {\em radial}, in that $k_\sigma(\bx, \bx') = g(\|\bx - \bx'\|^2)$ for some $g$.

In addition, we require that $k_\sigma$ be {\em positive semi-definite}, which means that
the matrix
$(k_\sigma(x_i,x_j))_{1 \le i,j \le m}$ is positive semi-definite for all
positive integers $m$ and all $x_1, \ldots, x_m \in \reals^d$. For radial kernels, this is equivalent to the condition that $g$ is completely monotone, i.e.,
\begin{gather*}
(-1)^k \frac{d^k}{dt^k} g(t) \geq 0, \quad \mbox{for all $k \ge 1, t > 0$},\\
\lim_{t \to 0} g(t) = g(0),
\end{gather*}
and to the assumption that there exists a finite Borel measure $\mu$ on $\reals^+ \triangleq [0, \infty)$ such that
\begin{equation*}
k_\sigma(\bx, \bx') = \int \exp\bigl(-t^2\|\bx - \bx'\|^2\bigr) d\mu(t).
\end{equation*}
See \citet{scovel10}. Well-known examples of kernels satisfying all of the above properties are the Gaussian kernel
\begin{equation}\label{eqn:gaussian}
k_\sigma(\bx, \bx') = \biggl(\frac{1}{\sqrt{2\pi}\sigma}\biggr)^{d} \exp\biggl(-\frac{\|\bx-\bx'\|^2}{2\sigma^2}\biggr),
\end{equation}
the multivariate Student kernel
\begin{equation*}
k_\sigma(\bx, \bx') = \biggl(\frac{1}{\sqrt{\pi}\sigma}\biggr)^d\cdot\frac{\Gamma\bigl((\nu+d)/2\bigr)}{\Gamma(\nu/2)}\cdot\biggl(1+\frac{1}{\nu}\cdot\frac{\|\bx-\bx'\|^2}{\sigma^2}\biggr)^{-\frac{\nu+d}{2}},
\end{equation*}
and the Laplacian kernel
\begin{equation*}
k_\sigma(\bx, \bx') = \frac{c_d}{\sigma^d}\exp\biggl(-\frac{\|\bx-\bx'\|}{\sigma}\biggr)
\end{equation*}
where $c_d$ is a constant depending on the dimension $d$ that ensures $\int k_\sigma\left(\bx, \,\cdot\,\right) \, d\bx= 1$. The PSD assumption does, however, exclude several common kernels for density estimation, including those with finite support.

It is possible to associate every PSD kernel with a feature map and a Hilbert space.  Although there are many ways to do this, we will consider the following canonical construction. Define $\Phi(\bx) \triangleq
k_\sigma(\cdot,\bx)$, which is called the {\em canonical feature map} associated
with $k_\sigma$. Then define the Hilbert space of functions ${\cal H}$ to be the completion of the span of $\{\Phi(\bx) \, : \, \bx \in \reals^d\}$. This space is known as the reproducing kernel Hilbert space (RKHS) associated with $k_\sigma$. See \citet{steinwart08} for a thorough treatment of PSD kernels and RKHSs. For our purposes, the critical property of ${\cal H}$ is the so-called {\em reproducing property}. It states that for all $g \in {\cal H}$ and all $\bx \in \reals^d$, $g(\bx) = \langle \Phi(\bx), g \rangle_{{\cal H}}$. As a special case, taking $g = k_\sigma( \cdot, \bx')$, we obtain
$$
k(\bx,\bx') = \langle \Phi(\bx), \Phi(\bx') \rangle
$$
for all $\bx, \bx' \in \reals^d$. Therefore, the kernel evaluates the inner product of its arguments after they have been transformed by $\Phi$.

For radial kernels, $\|\Phi(\bx)\|_\mathcal{H}$ is constant since
\begin{equation*}
\|\Phi(\bx)\|_\mathcal{H}^2 = \langle \Phi(\bx), \Phi(\bx) \rangle_\mathcal{H} = k_\sigma(\bx, \bx) = k_\sigma(\bzero, \bzero).
\end{equation*}
We will denote $\tau = \|\Phi(\bx)\|_\mathcal{H}$.

From this point of view, the KDE can be expressed as
\begin{align*}
\widehat{f}_{KDE}(\cdot) & = \frac{1}{n} \sum_{i = 1}^{n} k_\sigma(\cdot, \bX_i) \\
& = \frac{1}{n} \sum_{i = 1}^{n} \Phi(\bX_i),
\end{align*}
the sample mean of the $\Phi(\bX_i)$'s. Equivalently, $\widehat{f}_{KDE} \in {\cal H}$ is the solution of
\begin{equation*}
\min_{g \in \mathcal{H}} \sum_{i=1}^n \|\Phi(\bX_i) - g\|_{\mathcal{H}}^2.
\end{equation*}

Being the solution of a least squares problem, the KDE is sensitive to the presence of outliers among the $\Phi(\bX_i)$'s.
To reduce the effect of outliers, we propose to use $M$-estimation \citep{huber64} to find a robust sample mean of the $\Phi(\bX_i)$'s. For a robust loss function $\rho(x)$ on $x \geq 0$, the robust kernel density estimate is defined as
\begin{equation}\label{eqn:rkde}
\widehat{f}_{RKDE} = \argmin_{g \in \mathcal{H}}\sum_{i=1}^n \rho\bigl(\|\Phi(\bX_i) - g\|_{\mathcal{H}}\bigr).
\end{equation}
Well-known examples of robust loss functions are Huber's or Hampel's $\rho$. Unlike the quadratic loss, these loss functions have the property that
$\psi \triangleq \rho'$ is bounded. For Huber's $\rho$, $\psi$ is given by
\begin{equation}
\label{eqn:huber}
\psi\left(x\right) =
\begin{cases}
x, &\, 0 \leq x \leq a\\
a, &\, a < x.
\end{cases}
\end{equation}
and for Hampel's $\rho$,
\begin{equation}
\label{eqn:hampel}
\psi(x) =
\begin{cases}
x,& 0\leq x < a\\
a,& a\leq x < b\\
a\cdot(c-x)/(c-b),& b\leq x < c\\
0,& c\leq x.
\end{cases}
\end{equation}
The functions $\rho(x), \psi(x)$, and $\psi(x)/x$ are plotted in Figure
\ref{fig:rho_psi}, for the quadratic, Huber, and Hampel losses. Note that
while $\psi(x)/x$ is constant for the quadratic loss, for Huber's or
Hampel's loss, this function is decreasing in $x$. This is a desirable
property for a robust loss function, which will be explained later in
detail. While our examples and experiments employ Huber's and Hampel's
losses, many other losses can be employed.

\begin{figure}[!tb]
\centering
\begin{minipage}[htb]{0.49\linewidth}
    \centering
    \includegraphics[width=1.0\linewidth]{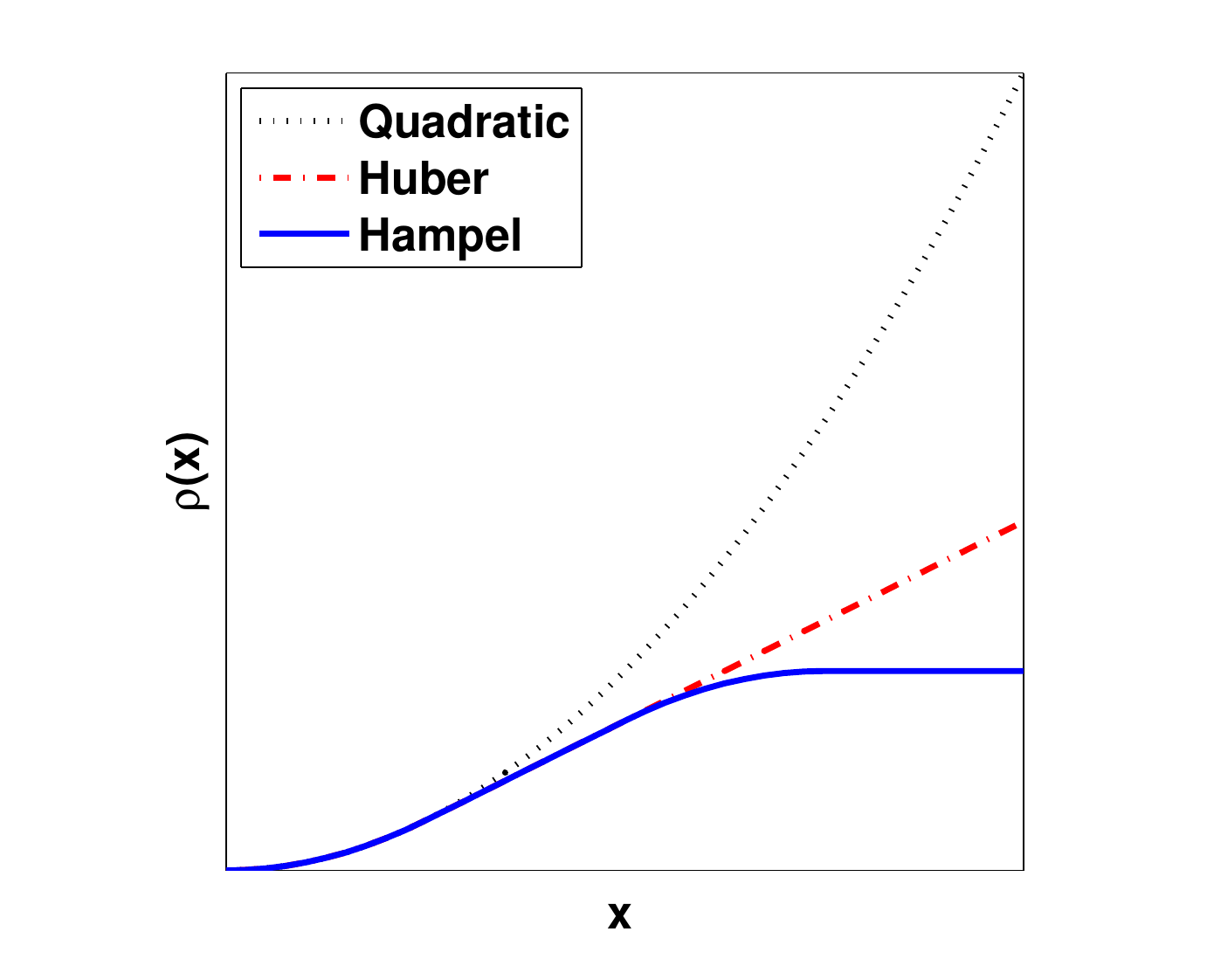}
    \centerline{(a) $\rho$ functions}
\end{minipage}
\hfill
\begin{minipage}[htb]{0.49\linewidth}
    \centering
    \includegraphics[width=1.0\linewidth]{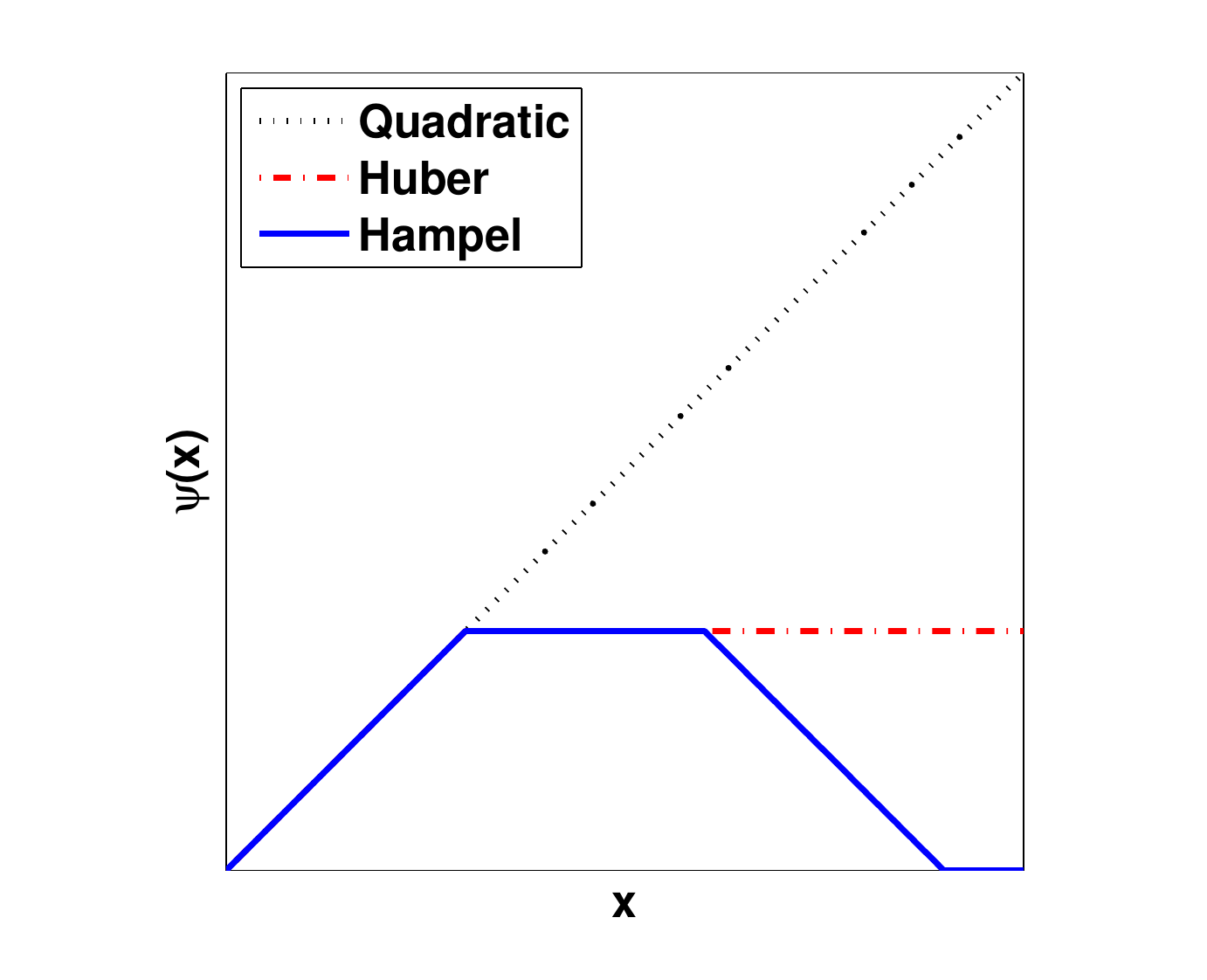}
    \centerline{(b) $\psi$ functions}
\end{minipage}
\begin{minipage}[htb]{0.49\linewidth}
    \centering
    \includegraphics[width=1.0\linewidth]{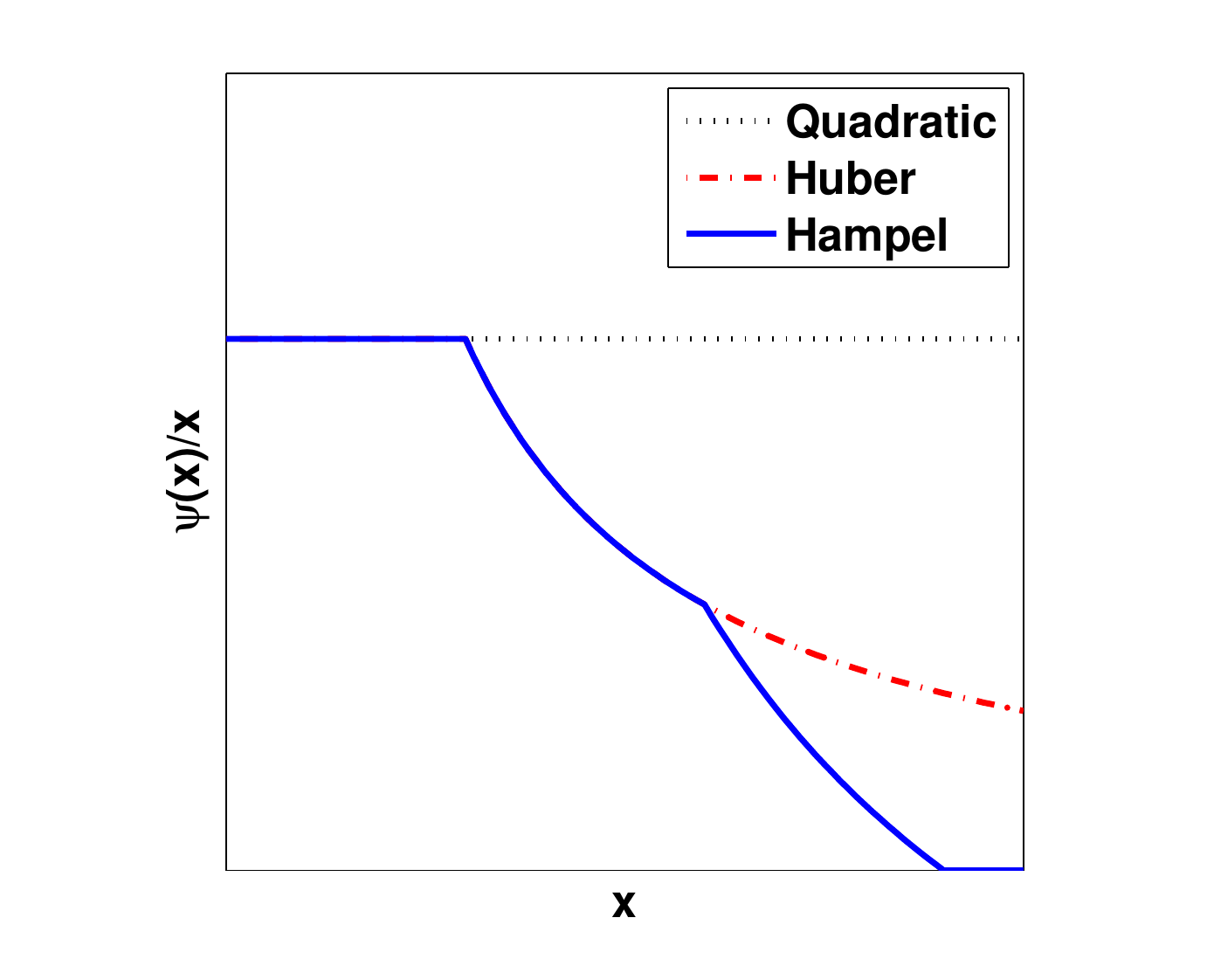}
    \centerline{(c) $\psi(x)/x$}
\end{minipage}
\caption{The comparison between three different $\rho(x)$, $\psi(x)$, and $\psi(x)/x$: quadratic, Huber's, and Hampel's.}\label{fig:rho_psi}
\end{figure}

We will argue below that $\widehat{f}_{RKDE}$ is a valid density, having
the form $\sum_{i=1}^n w_i k_\sigma(\cdot, \bX_i)$ with weights $w_i$ that
are nonnegative and sum to one. To illustrate the estimator, Figure
\ref{fig:2d} (a) shows a contour plot of a Gaussian mixture distribution
on $\reals^2$. Figure \ref{fig:2d} (b) depicts a contour plot of a KDE
based on a training sample of size $200$ from the Gaussian mixture. As we
can see in Figure \ref{fig:2d} (c) and (d), when $20$ contaminating data
points are added,
the KDE is significantly altered in low density regions, while the RKDE is
much less affected.

\begin{figure}[!tb]
\centering
\begin{minipage}[htb]{.49\linewidth}
  \centering
  \includegraphics[width=0.8\linewidth]{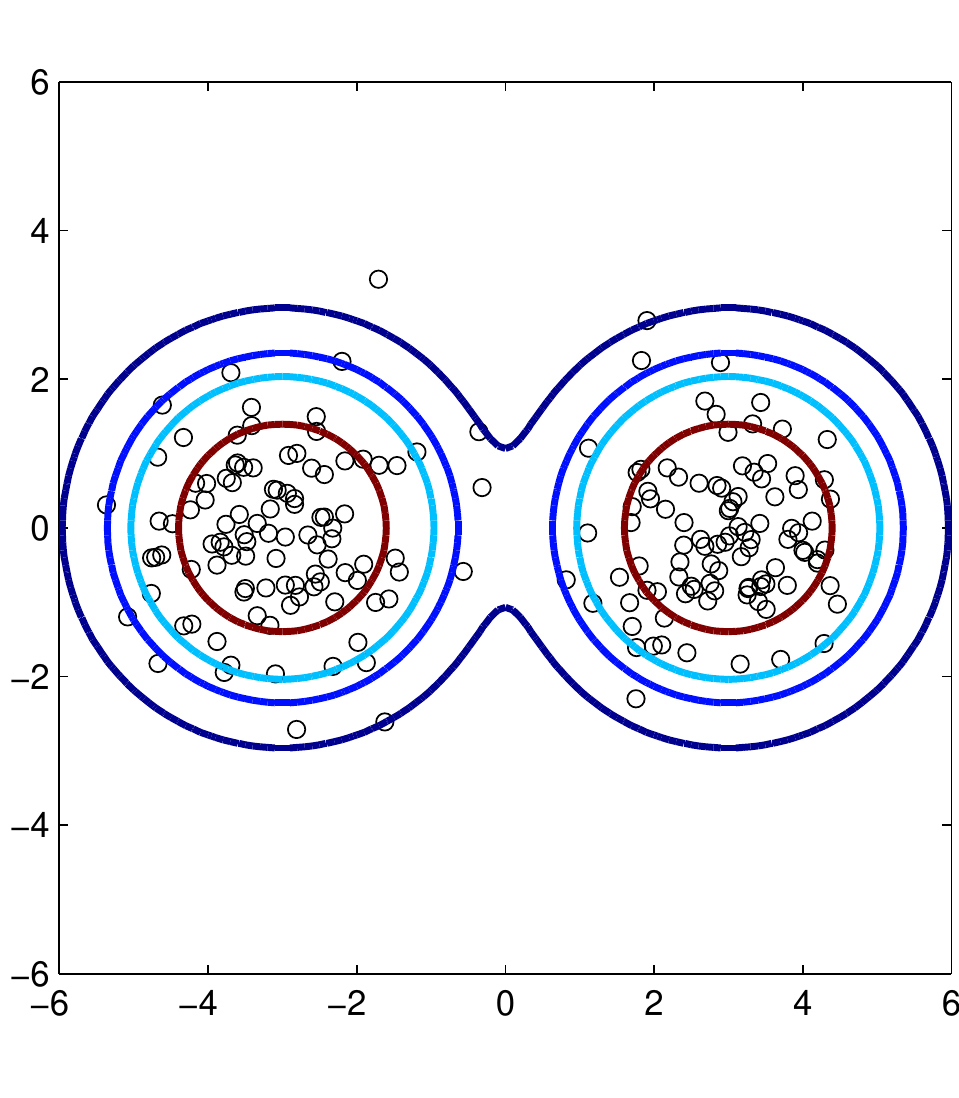}
  \centerline{(a) True density}\medskip
\end{minipage}
\hfill
\begin{minipage}[htb]{.49\linewidth}
  \centering
  \includegraphics[width=0.8\linewidth]{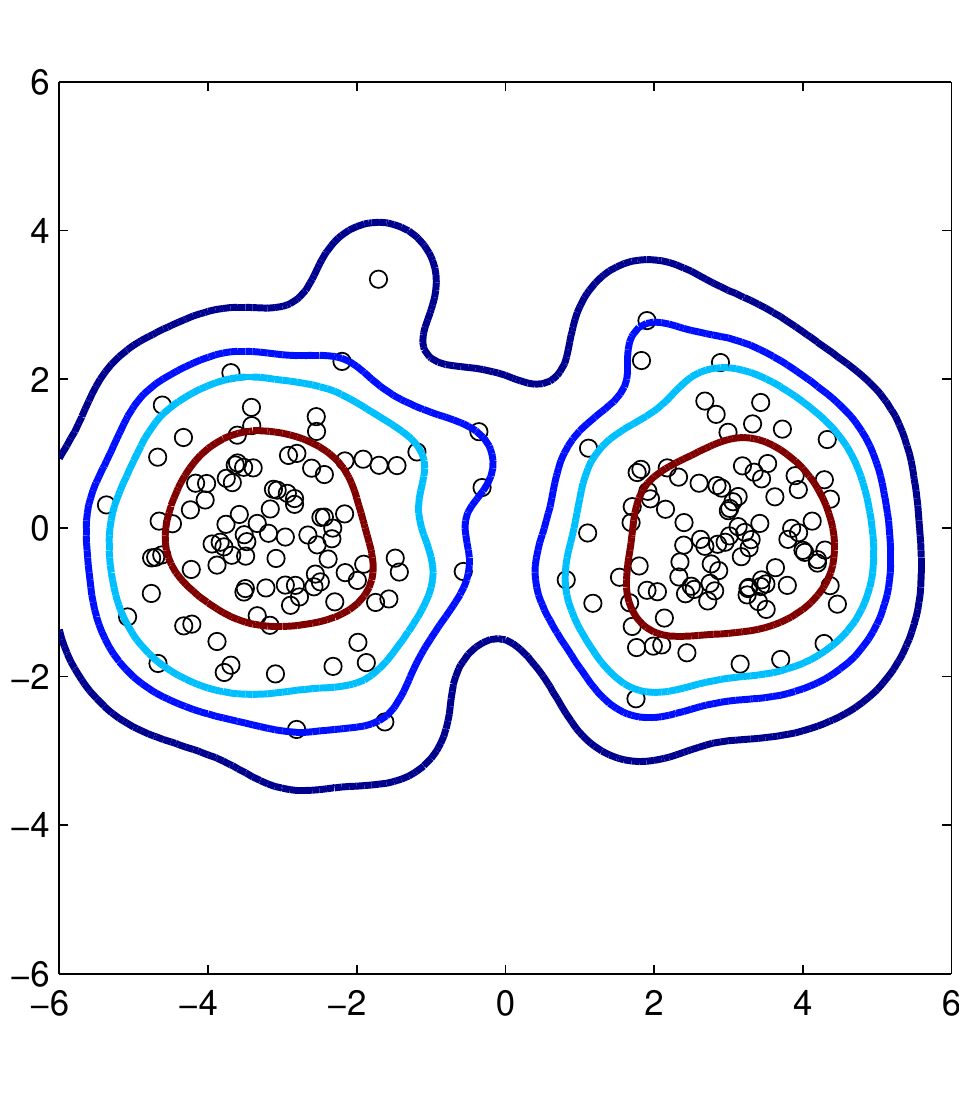}
  \centerline{(b) KDE without outliers}\medskip
\end{minipage}
\centering
\begin{minipage}[htb]{.49\linewidth}
  \centering
  \includegraphics[width=0.8\linewidth]{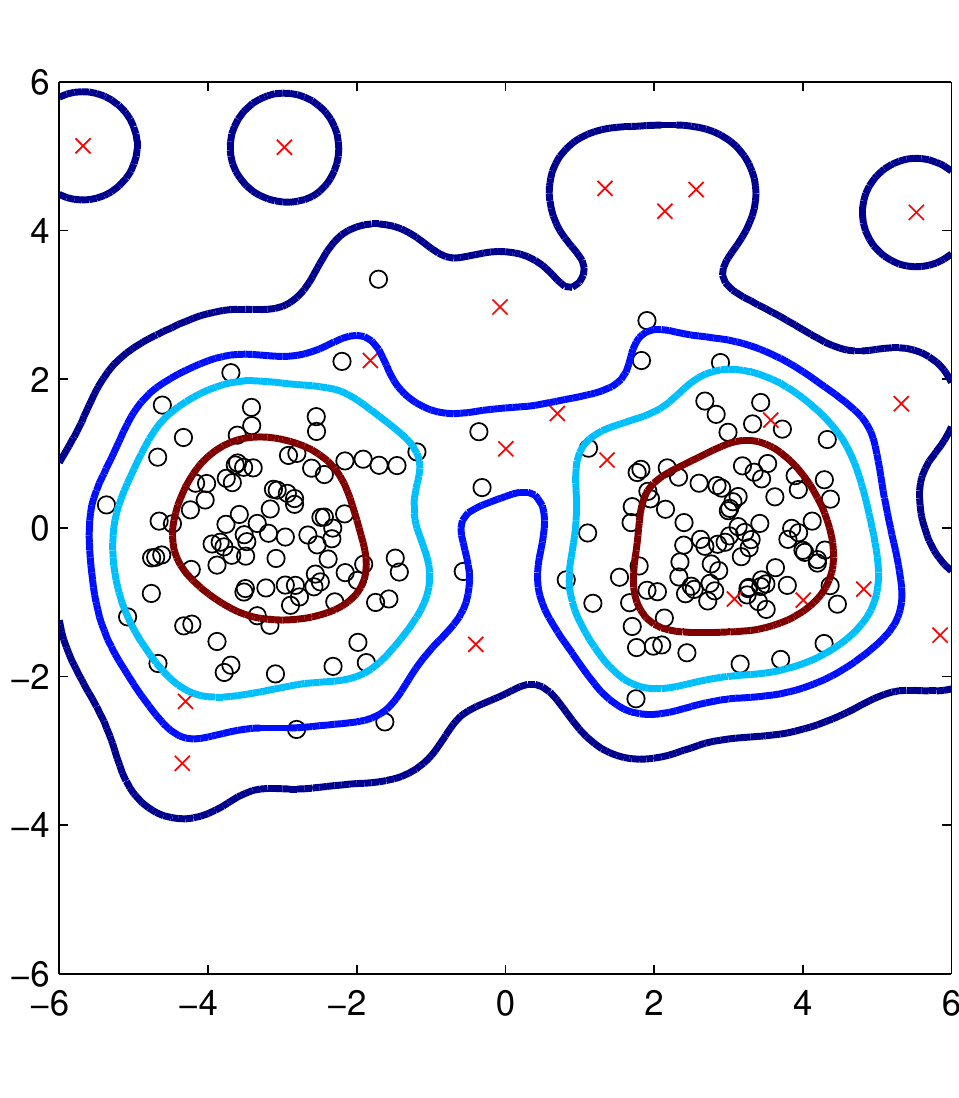}
  \centerline{(c) KDE with outliers}\medskip
\end{minipage}
\hfill
\begin{minipage}[htb]{.49\linewidth}
  \centering
  \includegraphics[width=0.8\linewidth]{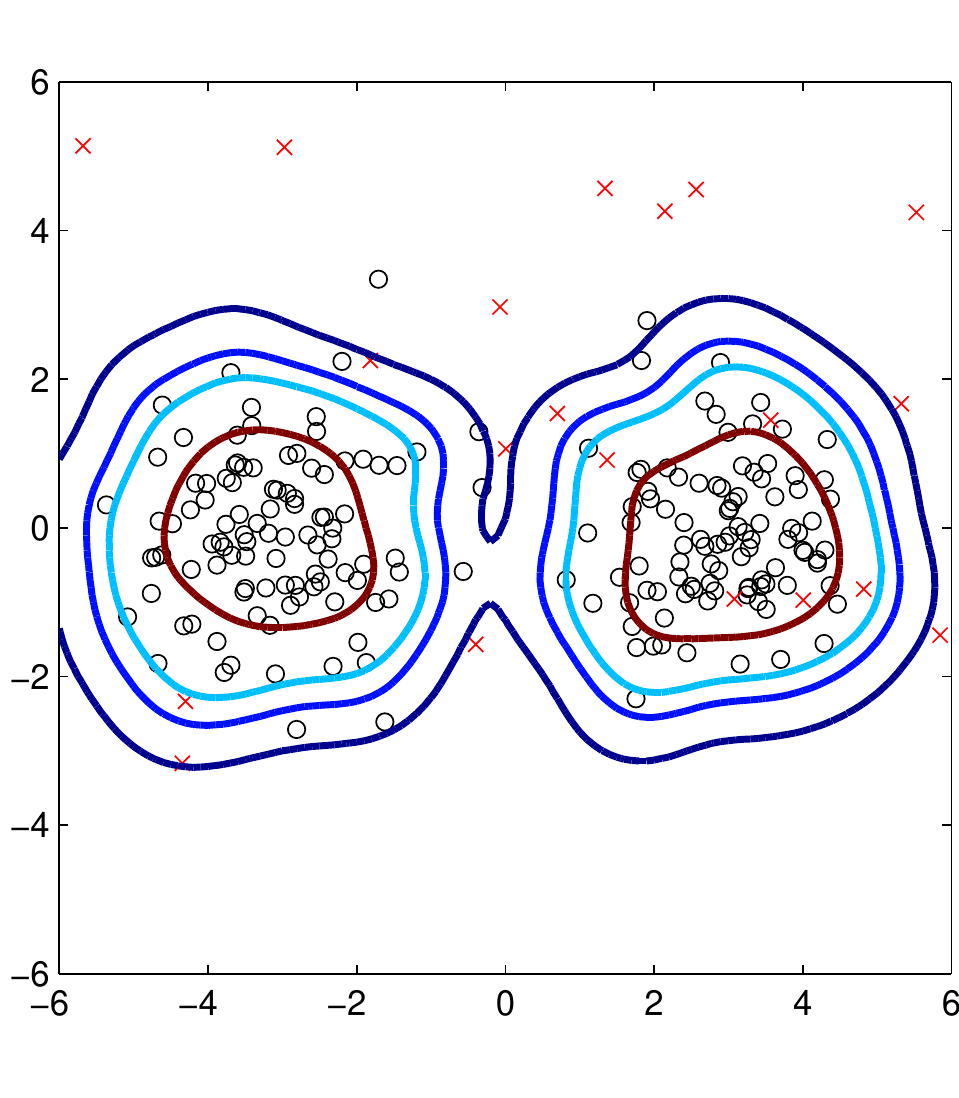}
  \centerline{(d) RKDE with outliers}\medskip
\end{minipage}
\caption{Contours of a nominal density and kernel density estimates along
with
data samples from the nominal density (o) and contaminating density (x).
$200$ points
are from the nominal distribution and $20$ contaminating points
are from a uniform
distribution.} \label{fig:2d} \end{figure}

Throughout this paper, we define $\varphi(x) \triangleq \psi(x)/x$ and consider the following assumptions on $\rho$, $\psi$, and $\varphi$:
\begin{itemize}
    \item [(A1)] $\rho$ is non-decreasing, $\rho(0) = 0$, and $\rho(x)/x \to 0$ as $x \to 0$
    \item [(A2)] $\varphi(0) \triangleq \lim_{x \to 0} \frac{\psi(x)}{x}$ exists and is finite
    \item [(A3)] $\psi$ and $\varphi$ are continuous
    \item [(A4)] $\psi$ and $\varphi$ are bounded
    \item [(A5)] $\varphi$ is Lipschitz continuous
\end{itemize}
which hold for Huber's and Hampel's losses, as well as several others.

\section{Representer Theorem}\label{sec:representer}

In this section, we will describe how $\widehat{f}_{RKDE}(\bx)$ can be
expressed as a weighted combination of the $k_\sigma(\bx,\bX_i)$'s. A
formula for the weights explains how a robust sample mean in ${\cal H}$
translates to a robust nonparametric density estimate. We also present
necessary and sufficient conditions for a function to be an RKDE. From
(\ref{eqn:rkde}), $\widehat{f}_{RKDE} = \argmin_{g \in \mathcal{H}} J(g)$,
where
\begin{equation}\label{eqn:optimization}
J(g) = \frac{1}{n} \sum_{i=1}^n \rho(\|\Phi(\bX_i) - g\|_\mathcal{H}).
\end{equation}

First, let us find necessary conditions for $g$ to be a minimizer of $J$. Since the space over which we are optimizing $J$ is a Hilbert space, the necessary conditions are characterized through Gateaux differentials of $J$. Given a vector space $\mathcal{X}$ and a function $T: \mathcal{X} \to \reals$, the Gateaux differential of $T$ at $x \in \mathcal{X}$ with incremental $h \in \mathcal{X}$ is defined as
\begin{equation*}
\delta T(x; h) = \lim_{\alpha \to 0} \frac{T(x+ \alpha h) - T(x)}{\alpha}.
\end{equation*}
If $\delta T(x_0; h)$ is defined for all $h \in \mathcal{X}$, a necessary condition for $T$ to have a minimum at $x_0$ is that $\delta T(x_0; h) = 0$ for all $h \in \mathcal{X}$ \citep{david97}. From this optimality principle, we have the following lemma.

\begin{lemma}\label{lemma:v_mu}
Suppose assumptions (A1) and (A2) are satisfied. Then the Gateaux differential of $J$ at $g \in \mathcal{H}$ with incremental $h \in \mathcal{H}$ is
\begin{align*}
\delta J(g ; h) = -\bigl\langle V(g), h \bigr \rangle_\mathcal{H}
\end{align*}
where $V: \mathcal{H} \to \mathcal{H}$ is given by
\begin{equation*}\label{eqn:V}
V(g) = \frac{1}{n} \sum_{i=1}^n \varphi(\|\Phi(\bX_i) - g\|_{\mathcal{H}}) \cdot \bigl(\Phi(\bX_i) - g\bigr).
\end{equation*}
A necessary condition for $g = \widehat{f}_{RKDE}$ is $V(g) = \bzero$.
\end{lemma}

Lemma \ref{lemma:v_mu} is used to establish the following representer theorem, so named because $\widehat{f}_{RKDE}$ can be represented as a weighted combination of kernels centered at the data points. Similar results are known for supervised kernel methods \citep{scholkopf01}.
\begin{theorem}\label{thm:representer}
Suppose assumptions (A1) and (A2) are satisfied. Then,
\begin{equation}\label{eqn:representer}
\widehat{f}_{RKDE}(\bx) = \sum_{i=1}^n w_i k_\sigma(\bx, \bX_i)
\end{equation}
where $w_i \geq 0$, $\sum_{i=1}^n w_i =1$. Furthermore,
\begin{equation}\label{eqn:w_i_proto}
w_i \propto \varphi(\|\Phi(\bX_i) - \widehat{f}_{RKDE}\|_{\mathcal{H}}).
\end{equation}
\end{theorem}

It follows that $\widehat{f}_{RKDE}$ is a density. The representer theorem also gives the following interpretation of the RKDE. If $\varphi$ is decreasing, as is the case for a robust loss, then $w_i$ will be small when $\|\Phi(\bX_i) - \widehat{f}_{RKDE}\|_{\mathcal{H}}$ is large. Now for any $g \in \mathcal{H}$,
\begin{align*}
\|\Phi(\bX_i) - g\|_{\mathcal{H}}^2 & = \langle\Phi(\bX_i) - g, \Phi(\bX_i) - g \rangle_\mathcal{H}\\
& = \|\Phi(\bX_i)\|_{\mathcal{H}}^2 - 2 \langle \Phi(\bX_i), g \rangle_\mathcal{H} + \|g\|_{\mathcal{H}}^2\\
& = \tau^2 - 2 g(\bX_i) + \|g\|_{\mathcal{H}}^2.
\end{align*}
Taking $g = \widehat{f}_{RKDE}$, we see that $w_i$ is small when $\widehat{f}_{RKDE}(\bX_i)$ is small. Therefore, the RKDE is robust in the sense that it down-weights outlying points.

Theorem \ref{thm:representer} provides a necessary condition for $\widehat{f}_{RKDE}$ to be the minimizer of (\ref{eqn:optimization}). With an additional assumption on $J$, this condition is also sufficient.
\begin{theorem}\label{thm:sufficient}
Suppose that assumptions (A1) and (A2) are satisfied, and $J$ is strictly convex. Then (\ref{eqn:representer}), (\ref{eqn:w_i_proto}), and $\sum_{i=1}^n w_i = 1$ are sufficient for $\widehat{f}_{RKDE}$ to be the minimizer of (\ref{eqn:optimization}).
\end{theorem}

Since the previous result assumes $J$ is strictly convex, we give some simple conditions that imply this property.
\begin{lemma}\label{lemma:strict_convex}
$J$ is strictly convex provided either of the following conditions is satisfied:
\begin{itemize}
\item [(i)] $\rho$ is strictly convex and non-decreasing.
\item [(ii)] $\rho$ is convex, strictly increasing, $n \geq 3$, and $K = (k_\sigma(\bX_i, \bX_j))_{i, j=1}^n$ is positive definite.
\end{itemize}
\end{lemma}
The second condition implies that $J$ can be strictly convex even for the Huber loss, which is convex but not strictly convex.

\section{KIRWLS Algorithm and Its Convergence}\label{sec:kirwls}
In general, (\ref{eqn:rkde}) does not have a closed form solution and $\widehat{f}_{RKDE}$ has to be found by an iterative algorithm. Fortunately, the iteratively re-weighted least squares (IRWLS) algorithm used in classical $M$-estimation \citep{huber64} can be extended to a RKHS using the \emph{kernel trick}. The kernelized iteratively re-weighted least squares (KIRWLS) algorithm starts with initial $w_i^{(0)} \in \reals$ , $i = 1, \dots, n$ such that $w_i^{(0)} \geq 0$ and $\sum_{i=1}^n w_i^{(0)} = 1$, and generates a sequence $\{f^{(k)}\}$ by iterating on the following procedure:
\begin{gather*}
f^{(k)} = \sum_{i=1}^n w_i^{(k-1)}\Phi(\bX_i) ,\\
w_i^{(k)} = \frac{\varphi(\|\Phi(\bX_i) - f^{(k)}\|_\mathcal{H})}{\sum_{j=1}^n\varphi(\|\Phi(\bX_j) - f^{(k)}\|_\mathcal{H})}.
\end{gather*}
Intuitively, this procedure is seeking a fixed point of equations (\ref{eqn:representer}) and (\ref{eqn:w_i_proto}).
The computation of $\|\Phi(\bX_j) - f^{(k)}\|_\mathcal{H}$ can be done by observing
\begin{align*}
\|\Phi(\bX_j) - f^{(k)}\|_\mathcal{H}^2  & = \left\langle \Phi(\bX_j) -
f^{(k)},
\Phi(\bX_j) - f^{(k)} \right\rangle_\mathcal{H}\\
& = \big\langle\Phi(\bX_j), \Phi(\bX_j)\big\rangle_\mathcal{H} -2 \big\langle\Phi(\bX_j), f^{(k)}\big\rangle_\mathcal{H} + \big\langle f^{(k)}, f^{(k)}\big\rangle_\mathcal{H}.
\end{align*}
Since $f^{(k)} = \sum_{i=1}^n w_i^{(k-1)}\Phi(\bX_i)$, we have
\begin{eqnarray*}
\big\langle\Phi(\bX_j), \Phi(\bX_j)\big\rangle_\mathcal{H} & = & k_\sigma(\bX_j, \bX_j)\\
\big\langle\Phi(\bX_j), f^{(k)}\big\rangle_\mathcal{H} & = & \sum_{i=1}^n w_i^{(k-1)} k_\sigma(\bX_j,
\bX_i)\\
\big\langle f^{(k)}, f^{(k)}\big\rangle_\mathcal{H} & = & \sum_{i=1}^n \sum_{l=1}^n w_i^{(k-1)} w_l^{(k-1)}k_\sigma(\bX_i, \bX_l).
\end{eqnarray*}
Recalling that $\Phi(\bx) = k_\sigma(\cdot,\bx)$, after the $k$th
iteration
$$
f^{(k)}(\bx) = \sum_{i=1}^n w_i^{(k-1)} k_\sigma\left(\bx, \bX_i\right).
$$
Therefore, KIRWLS produces a sequence of weighted KDEs. The computational
complexity is $O(n^2)$ per iteration. In our experience, the number of
iterations needed is typically well below $100$. Initialization is discussed in the experimental study below.

KIRWLS can also be viewed as a kind of optimization transfer/majorize-minimize algorithm \citep{Lange00,jacobson07} with a quadratic surrogate for $\rho$. This perspective is used in our analysis in Section \ref{ssec:proof_thm_convergence_bm}, where $f^{(k)}$ is seen to be the solution of a weighted least squares problem.

The next theorem characterizes the convergence of KIRWLS in terms of $\{J(f^{(k)})\}_{k=1}^\infty$ and $\{f^{(k)}\}_{k=1}^\infty$.
\begin{theorem}\label{thm:convergence_bm}
Suppose assumptions (A1) - (A3) are satisfied, and $\varphi(x)$ is nonincreasing. Let
\begin{equation*}
\mathcal{S} = \bigl\{g \in \mathcal{H}\,  \bigl | V(g) = \bzero \bigr\}
\end{equation*}
and $\{f^{(k)}\}_{k=1}^\infty$ be the sequence produced by the KIRWLS algorithm. Then, $J(f^{(k)})$ monotonically decreases at every
iteration and converges. Also, $\mathcal{S} \neq \emptyset$ and
\begin{equation*}
\|f^{(k)} - \mathcal{S}\|_{\mathcal{H}} \triangleq \inf_{g\in \mathcal{S}} \|f^{(k)} - g\|_\mathcal{H} \to 0
\end{equation*}
as $k \to \infty$.
\end{theorem}
In words, as the number of iterations grows, $f^{(k)}$ becomes arbitrarily close to the set of stationary points of $J$,
points $g \in \mathcal{H}$ satisfying $\delta J(g; h) = 0 \quad \forall h \in \mathcal{H}$.

\begin{corollary}
Suppose that the assumptions in Theorem \ref{thm:convergence_bm} hold and $J$ is strictly convex. Then, $\{f^{(k)}\}_{k=1}^\infty$ converges to $\widehat{f}_{RKDE}$ in the $\mathcal{H}$-norm.
\end{corollary}
This follows because under strict convexity of $J$, $|\mathcal{S}| = 1$.

\section{Influence Function for Robust KDE}\label{sec:ic}
To quantify the robustness of the RKDE, we study the influence function. First, we recall the traditional influence function from
robust statistics. Let $T(F)$ be an estimator of a scalar parameter based on a distribution $F$. As a measure of robustness of $T$, the influence function was proposed by
\citet{hampel74}. The influence function (IF) for $T$ at $F$ is defined as
\begin{equation*}
IF(x'; T, F) = \lim_{s\to 0} \frac{T((1-s)F + s\delta_{x'}) - T(F)}{s},
\end{equation*}
where $\delta_{x'}$ represents a discrete distribution that assigns probability $1$ to the point $x'$. Basically, $IF(x'; T, F)$ represents how $T(F)$ changes when the distribution $F$ is contaminated with infinitesimal probability mass at $x'$. One robustness measure of $T$ is whether the corresponding IF is bounded or not.

For example, the maximum likelihood estimator for the unknown mean $\theta$ of Gaussian distribution is the sample mean $T(F)$,
\begin{equation}\label{eqn:T}
T(F) = E_F [X] = \int x \, dF(x).
\end{equation}
The influence function for $T(F)$ in (\ref{eqn:T}) is
\begin{align*}
IF(x'; T, F) & = \lim_{s\to 0} \frac{T((1-s)F + s\delta_{x'}) - T(F)}{s}\\
& = x' - E_F[X].
\end{align*}
Since $|IF(x'; T, F)|$ increases without bound as $x'$ goes to $\pm \infty$, the estimator is considered to be not robust.

Now, consider a similar concept for a function estimate. Since the estimate is a function, not a scalar, we should be able to express
the change of the function value at every $\bx$.
\begin{definition}[IF for function estimate]
Let $T(\bx; F)$ be a function estimate based on $F$, evaluated at $\bx$. We define the influence function for $T(\bx; F)$ as
\begin{equation*}
IF(\bx, \bx'; T, F) = \lim_{s\to 0} \frac{T(\bx; F_s) - T(\bx; F)}{s}
\end{equation*}
where $F_s = (1-s)F+ s\delta_{\bx'}$.
\end{definition}
$IF(\bx, \bx'; T, F)$ represents the change of the estimated function $T$ at $\bx$ when we add infinitesimal probability mass at $\bx'$ to $F$. For example, the standard KDE is
\begin{align*}
T(\bx; F) & = \widehat{f}_{KDE}(\bx; F) = \int k_\sigma(\bx, \by) dF(\by)\\
& = E_F [k_\sigma(\bx,\bX)]
\end{align*}
where $\bX \sim F$. In this case, the influence function is
\begin{align}\label{eqn:ICKDE}
\nonumber IF(\bx, \bx'; \widehat{f}_{KDE}, F) &= \lim_{s\to 0} \frac{\widehat{f}_{KDE}(\bx;F_s) - \widehat{f}_{KDE}(\bx; F)}{s}\\
\nonumber & =  \lim_{s\to 0} \frac{E_{F_s} [k_\sigma(\bx,\bX)] - E_F [k_\sigma(\bx,\bX)]}{s}\\
\nonumber & =  \lim_{s\to 0} \frac{-sE_{F} [k_\sigma(\bx,\bX)] + sE_{\delta_{\bx'}} [k_\sigma(\bx,\bX)]}{s}\\
\nonumber & =  -E_{F} [k_\sigma(\bx,\bX)] + E_{\delta_{\bx'}} [k_\sigma(\bx,\bX)]\\
& =  -E_{F} [k_\sigma(\bx,\bX)] + k_\sigma(\bx,\bx')
\end{align}
With the empirical distribution $F_n = \frac{1}{n} \sum_{i=1}^n \delta_{\bX_i}$,
\begin{equation}\label{eqn:emICKDE}
IF(\bx, \bx'; \widehat{f}_{KDE}, F_n) = -\frac{1}{n}\sum_{i=1}^n k_\sigma(\bx, \bX_i) + k_\sigma(\bx, \bx').
\end{equation}

To investigate the influence function of the RKDE, we generalize its definition to a general distribution $\mu$, writing $\widehat{f}_{RKDE}(\,\cdot\,; \mu) = f_\mu$ where
\begin{equation*}
f_\mu = \argmin_{g \in \mathcal{H}} \int \rho(\|\Phi(\bx) - g\|_\mathcal{H}) \, d\mu(\bx).
\end{equation*}
For the robust KDE, $T(\bx, F) = \widehat{f}_{RKDE}(\bx; F) = \langle \Phi(\bx), f_{F} \rangle_\mathcal{H}$, we have the following characterization of the
influence function. Let $q(x) = x\psi'(x) - \psi(x)$.
\begin{theorem}\label{thm:influence_true}
Suppose assumptions (A1)-(A5) are satisfied. In addition, assume that $f_{F_s} \to f_{F}$ as $s \to 0$. If $\dot{f}_{F} \triangleq \lim_{s \to 0}\frac{f_{F_s}-f_{F}}{s}$ exists, then
\begin{equation*}
IF(\bx, \bx'; \widehat{f}_{RKDE}, F) = \langle \Phi(\bx), \dot{f}_{F} \rangle_\mathcal{H}
\end{equation*}
where $\dot{f}_{F} \in \mathcal{H}$ satisfies
\begin{eqnarray}\label{eqn:icfinal}
\nonumber \lefteqn{\biggl(\int \varphi(\|\Phi(\bx) - f_{F}\|_{\mathcal{H}}) dF \biggr)\cdot \dot{f}_{F} }\\
\nonumber & + & \int \biggl(\frac{\bigl\langle \dot{f}_{F} , \Phi(\bx) - f_{F} \bigr\rangle_\mathcal{H}}{\|\Phi(\bx) - f_{F}\|_{\mathcal{H}}^3}
\cdot q(\|\Phi(\bx) - f_{F}\|_{\mathcal{H}})\cdot\bigl(\Phi(\bx) - f_{F}\bigr)\biggr) dF(\bx)\\
& = & (\Phi(\bx') - f_{F})\cdot \varphi(\|\Phi(\bx') - f_{F}\|_{\mathcal{H}}).
\end{eqnarray}
\end{theorem}

Unfortunately, for Huber or Hampel's $\rho$, there is no closed form solution for $\dot{f}_{F}$ of (\ref{eqn:icfinal}). However,
if we work with $F_n$ instead of $F$, we can find $\dot{f}_{F_n}$ explicitly. Let
\begin{gather*}
\bones  = [1, \dots, 1]^T,\\
\bk'  = [k_\sigma(\bx', \bX_1), \dots, k_\sigma(\bx', \bX_n)]^T,
\end{gather*}
$I_n$ be the $n \times n$ identity matrix, $K\triangleq (k_\sigma(\bX_i, \bX_j))_{i=1, j=1}^n$ be the kernel matrix, $Q$ be a diagonal matrix with
$Q_{ii} = q(\|\Phi(\bX_i) - f_{F_n}\|_{\mathcal{H}})/\|\Phi(\bX_i) - f_{F_n}\|_{\mathcal{H}}^3$,
\begin{equation*}
\gamma = \sum_{i=1}^n \varphi(\|\Phi(\bX_i) - f_{F_n}\|_{\mathcal{H}}),
\end{equation*}
and
\begin{equation*}
\bw  = [w_1, \dots, w_n]^T,
\end{equation*}
where $\bw$ gives the RKDE weights as in (\ref{eqn:representer}).

\begin{theorem}\label{thm:influence_emp}
Suppose assumptions (A1)-(A5) are satisfied. In addition, assume that
\begin{itemize}
\item $f_{F_{n,s}} \to f_{F_n}$ as $s\to 0$ (satisfied when $J$ is strictly convex)
\item the extended kernel matrix $K'$ based on $\{\bX_i\}_{i=1}^n\bigcup\{\bx'\}$ is positive definite.
\end{itemize}
Then,
\begin{equation*}
IF(\bx, \bx'; \widehat{f}_{RKDE}, F_n) = \sum_{i=1}^n \alpha_i k_\sigma(\bx, \bX_i) + \alpha' k_\sigma(\bx, \bx')
\end{equation*}
where
\begin{equation*}
\alpha' = n\cdot\varphi(\|\Phi(\bx') - f_{F_n}\|_{\mathcal{H}})/\gamma
\end{equation*}
and $\balpha = [\alpha_1, \dots, \alpha_n]^T$ is the solution of the following system of linear equations:
\begin{eqnarray*}
\lefteqn{\biggl\{\gamma I_n + (I_n -  \bones\cdot \bw^T)^T Q (I_n - \bones \cdot\bw^T)K\biggr\} \balpha}\\
& = &{}- n \varphi(\|\Phi(\bx') - f_{F_n}\|_{\mathcal{H}}) \bw -\alpha'  (I_n - \bones\cdot\bw^T)^T Q\cdot(I_n - \bones \cdot\bw^T)  \cdot \bk'.
\end{eqnarray*}
\end{theorem}
Note that $\alpha'$ captures the amount by which the density estimator
changes near $\bx'$ in response to contamination at $\bx'$. Now $\alpha'$
is given by
\begin{equation*}
\alpha' = \frac{\varphi(\|\Phi(\bx') - f_{F_n}\|_{\mathcal{H}})}{\frac{1}{n}\sum_{i=1}^n \varphi(\|\Phi(\bX_i) - f_{F_n}\|_{\mathcal{H}})}.
\end{equation*}
For a standard KDE, we have $\varphi \equiv 1$ and $\alpha' = 1$, in agreement with (\ref{eqn:emICKDE}). For robust $\rho$, $\varphi(\|\Phi(\bx') - f_{F_n}\|_{\mathcal{H}})$ can be viewed as a measure of ``inlyingness'', with more inlying points having larger values. This follows from the discussion just after Theorem \ref{thm:representer}. If the contaminating point $\bx'$ is less inlying than the average $\bX_i$, then $\alpha' < 1$. Thus, the RKDE is less sensitive to outlying points than the KDE.

As mentioned above, in classical robust statistics, the robustness of an estimator can be inferred from the boundedness of the corresponding
influence function. However, the influence functions for density estimators are bounded even if $\|\bx'\| \to \infty$. Therefore, when we
compare the robustness of density estimates, we compare how close the influence functions are to the zero function.

\begin{figure}[!tb]
\centering
\begin{minipage}[htb]{0.49\linewidth}
    \centering
    \includegraphics[width=1.0\linewidth]{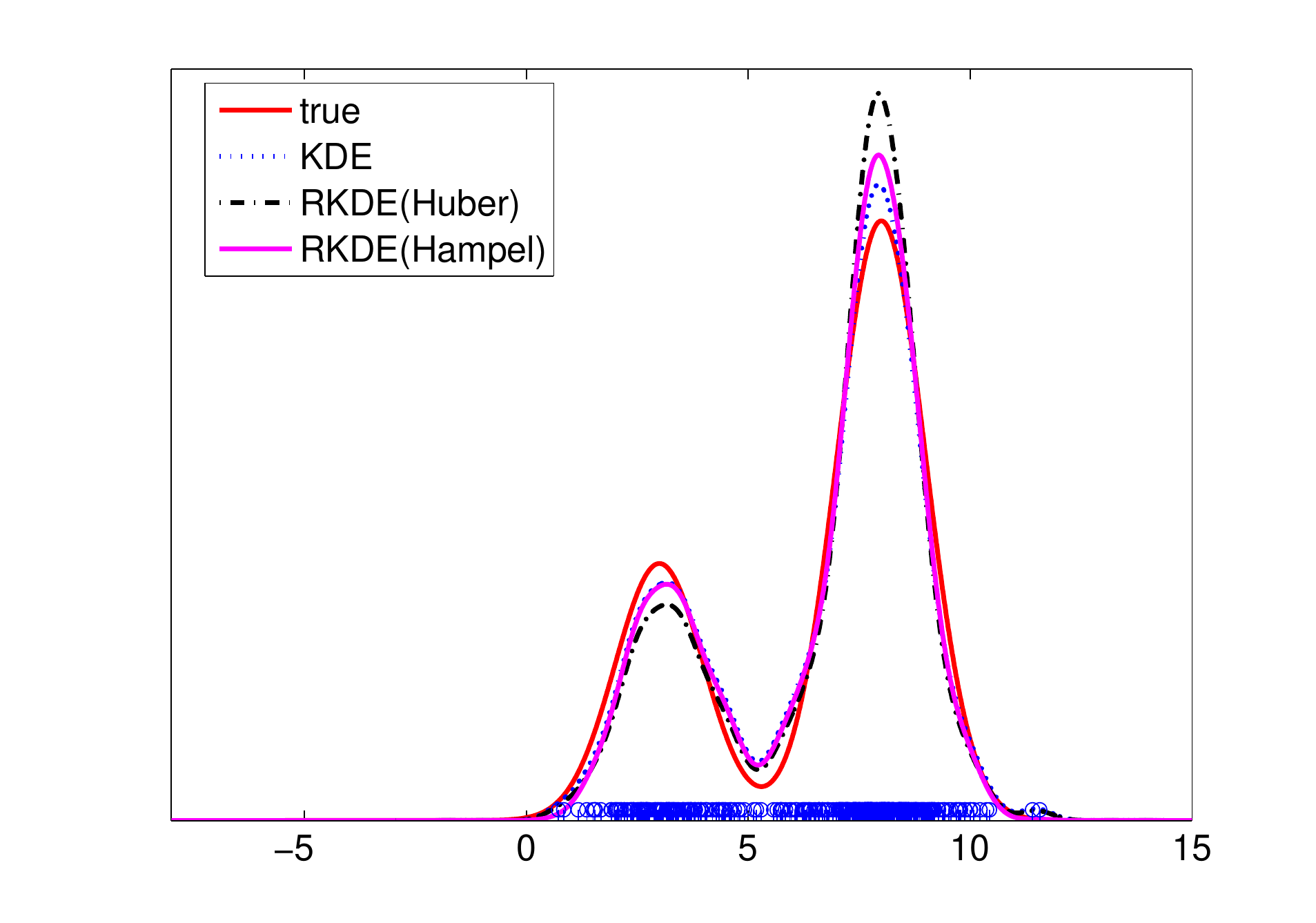}
    \centerline{\small (a)}
\end{minipage}
\hfill
\begin{minipage}[htb]{0.49\linewidth}
    \centering
    \includegraphics[width=1.0\linewidth]{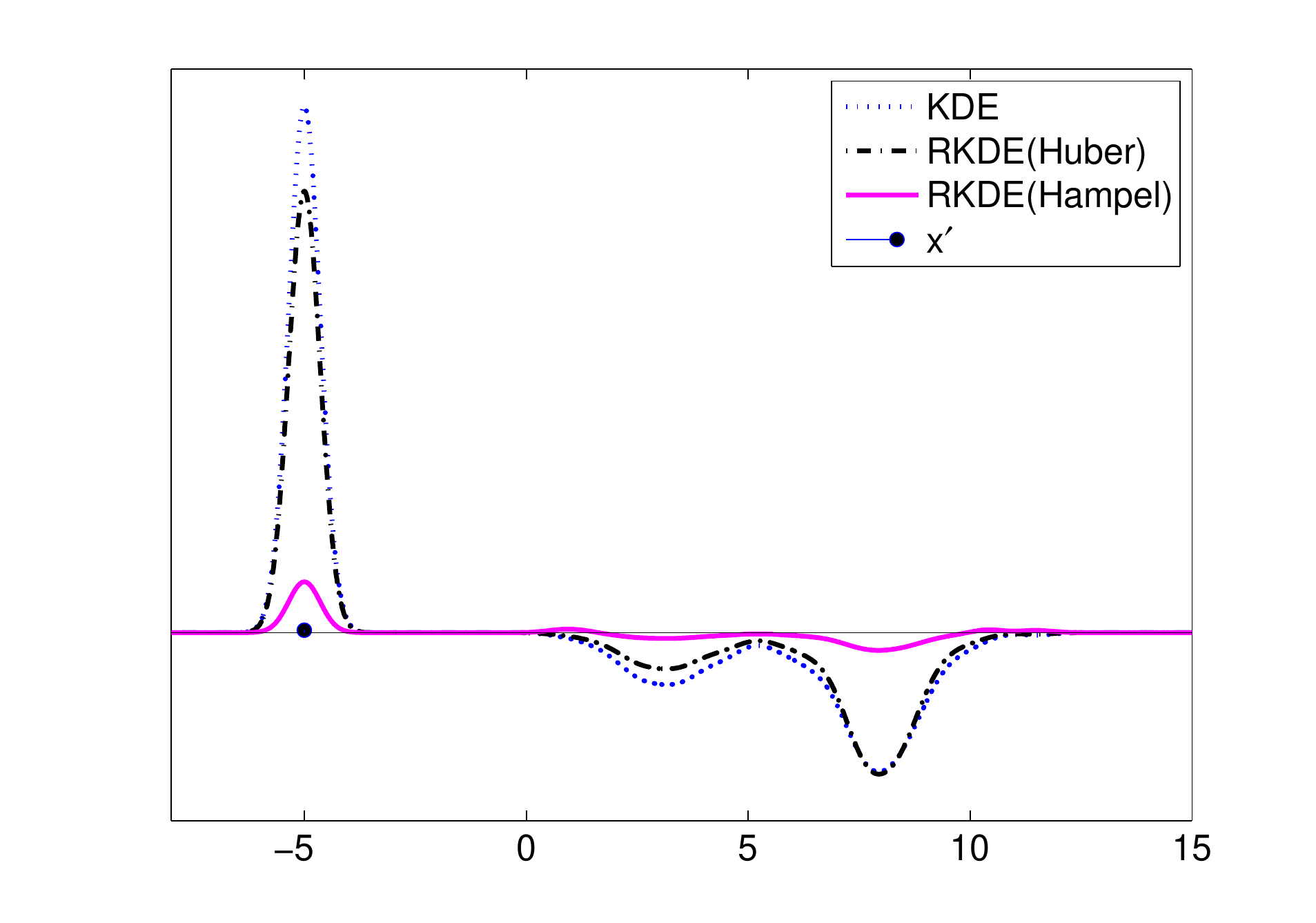}
    \centerline{\small (b)}
\end{minipage}
\caption{(a) true density and density estimates. (b) IF as a function of $\bx$ when $\bx' = -5$}\label{fig:IC}
\end{figure}

Simulation results are shown in Figure \ref{fig:IC} for a synthetic univariate distribution. Figure \ref{fig:IC} (a) shows the density of the distribution, and three estimates. Figure \ref{fig:IC} (b) shows the corresponding influence functions. As we can see in (b), for a point $\bx'$ in the tails of $F$, the influence functions for the robust KDEs are overall smaller, in absolute value, than those of the standard KDE (especially with Hampel's loss). Additional numerical results are given in Section \ref{ssec:exp_results}.

Finally, it is interesting to note that for any density estimator $\widehat{f}$,
\begin{equation*}
\int IF(\bx, \bx'; \widehat{f}, F)\, d\bx = \lim_{s\to 0} \frac{\int \widehat{f}(\bx; F_s)\, d\bx- \int \widehat{f}(\bx; F)\, d\bx}{s} = 0.
\end{equation*}
Thus $\alpha' = -\sum_{i=1}^n \alpha_i$ for a robust KDE. This suggests that since $\widehat{f}_{RKDE}$ has a smaller increase at $\bx'$ (compared to the KDE), it will also have a smaller decrease (in absolute value) near the training data. Therefore, the norm of $IF(\bx, \bx'; \widehat{f}_{RKDE}, F_n)$ should be smaller overall when $\bx'$ is an outlier. We confirm this in our experiments below.

\section{Experiments}\label{sec:experiment}
The experimental setup is described in \ref{ssec:exp_setup}, and results are presented in \ref{ssec:exp_results}.
\subsection{Experimental Setup}\label{ssec:exp_setup}
Data, methods, and evaluation are now discussed.
\subsubsection{Data}
We conduct experiments on $15$ benchmark data sets (Banana, B. Cancer, Diabetes, F. Solar, German, Heart, Image, Ringnorm, Splice,
Thyroid, Twonorm, Waveform, Pima Indian, Iris, MNIST), which were originally used in the task of classification. The data sets are available
online: see http://www.fml.tuebingen.mpg.de/Members/ for the first $12$ data sets and the UCI machine learning repository for the last $3$ data sets.
There are 100 randomly permuted partitions of each data set into ``training'' and ``test'' sets (20 for Image, Splice, and MNIST).

Given $\bX_1, \dots, \bX_n \sim f = (1-p)\cdot f_0 + p \cdot f_1$, our goal is to estimate $f_0$, or the
level sets of $f_0$. For each data set with two classes, we take one class as the nominal data from $f_0$ and the other class as contamination from $f_1$. For Iris, there are $3$ classes and we take one class as nominal data and the other two as contamination. For MNIST, we choose to use digit $0$ as nominal and digit $1$ as contamination. For MNIST, the original dimension $784$ is reduced to $8$ via kernel PCA using a Gaussian kernel with bandwidth $30$. For each data set, the training sample consists of $n_0$ nominal data and $n_1$ contaminating points, where $n_1 =
\epsilon\cdot n_0$ for $\epsilon = 0$, $0.05$, $0.10$, $0.15$, $0.20$, $0.25$ and $0.30$. Note that each $\epsilon$ corresponds to an anomaly
proportion $p$ such that $p = \frac{\epsilon}{1+\epsilon}$. $n_0$ is always taken to be the full amount of training data for the nominal class.

\subsubsection{Methods}
In our experiments, we compare three density estimators: the standard kernel density estimator (KDE), variable kernel density estimator (VKDE), and robust kernel density estimator (RKDE) with Hampel's loss. For all methods, the Gaussian kernel in (\ref{eqn:gaussian}) is used as the kernel function $k_\sigma$ and the kernel bandwidth $\sigma$ is set as the median distance of a training point $\bX_i$ to its nearest neighbor.

The VKDE has a variable bandwidth for each data point,
\begin{equation*}
\widehat{f}_{VKDE}(\bx) = \frac{1}{n} \sum_{i=1}^n k_{\sigma_i}(\bx, \bX_i),
\end{equation*}
and the bandwidth $\sigma_i$ is set as
\begin{equation*}
\sigma_i = \sigma \cdot \biggl(\frac{\eta}{\widehat{f}_{KDE}(\bX_i)}\biggr)^{1/2}
\end{equation*}
where $\eta$ is the mean of $\{\widehat{f}_{KDE}(\bX_i)\}_{i=1}^n$ \citep{abramson82, meer01}. There is another implementation of the VKDE where $\sigma_i$ is based on the distance to its $k$-th nearest neighbor \citep{breiman77}. However, this version did not perform as well and is therefore omitted.

For the RKDE, the parameters $a$, $b$, and $c$ in (\ref{eqn:hampel}) are set as follows. First, we compute $\widehat{f}_{med}$, the RKDE based on $\rho = |\,\cdot\,|$, and set $d_i = \|\Phi\left(\bX_i\right) - \widehat{f}_{med}\|_\mathcal{H}$. Then, $a$ is set to be the median of $\{d_i\}$, $b$ the $75$th percentile of $\{d_i\}$, and $c$ the $85$th percentile of $\{d_i\}$. After finding these parameters, we initialize $w_i^{(0)}$ such that $f^{(1)} = \widehat{f}_{med}$ and terminate KIRWLS when
\begin{equation*}
\frac{|J(f^{(k+1)}) - J(f^{(k)})|}{J(f^{(k)})} < 10^{-8}.
\end{equation*}

\subsubsection{Evaluation}
We evaluate the performance of the three density estimators in three different settings. First, we use the influence function to study sensitivity to outliers. Second and third, we compare the methods at the tasks of density estimation and anomaly detection, respectively. In each case, an appropriate performance measure is adopted. These are explained in detail in Section \ref{ssec:exp_results}. To compare a pair of methods across multiple data sets, we adopt the Wilcoxon signed-rank test \citep{wilcoxon45}. Given a performance measure, and given a pair of methods and $\epsilon$, we compute the difference $h_i$ between the performance of two density estimators on the $i$th data set. The data sets are ranked 1 through 15 according to their absolute values $|h_i|$, with the largest $|h_i|$ corresponding to the rank of 15. Let $R_1$ be the sum of ranks over these data sets where method 1 beats method 2, and let $R_2$ be the sum of the ranks for the other data sets. The signed-rank test statistic $T\triangleq \min(R_1, R_2)$ and the corresponding $p$-value are used to test whether the performances of the two methods are significantly different. For example, the critical value of $T$ for the signed rank test is $25$ at a significance level of $0.05$. Thus, if $T \leq 25$, the two methods are significantly different at the given significance level, and the larger of $R_1$ and $R_2$ determines the method with better performance.

\subsection{Experimental Results}\label{ssec:exp_results}
We begin by studying influence functions.
\subsubsection{Sensitivity using influence function}\label{sssec:exp_if}
As the first measure of robustness, we compare the influence functions for KDEs and RKDEs, given in (\ref{eqn:emICKDE}) and Theorem \ref{thm:influence_emp}, respectively. To our knowledge, there is no formula for the influence function of VKDEs, and therefore VKDEs are excluded in the comparison. We examine
$\alpha(\bx') = IF(\bx', \bx'; T, F_n)$ and
\begin{align*}
\beta(\bx') = \biggl(\int \bigl(IF(\bx, \bx'; T, F_n)\bigr)^2 d\bx\biggr)^{1/2}.
\end{align*}
In words, $\alpha(\bx')$ reflects the change of the density estimate value at an added point $\bx'$ and $\beta(\bx')$ is an overall impact of $\bx'$ on the density estimate over $\reals^d$.

In this experiment, $\epsilon$ is equal to 0, i.e, the density estimators are learned from a pure nominal sample. Then, we take contaminating points from the test sample, each of which serves as an $\bx'$. This gives us multiple $\alpha(\bx')$'s and $\beta(\bx')$'s. The performance measures are the medians of $\{\alpha(\bx')\}$ and $\{\beta(\bx')\}$ (smaller means better performance). The results using signed rank statistics are shown in Table \ref{tab:rank_if}. The results clearly states that for all data sets, RKDEs are less affected by outliers than KDEs.

\begin{table}[t]
\begin{center}
\begin{tabular}[c c c c c c] {c| c| c| c | c}
   \hline
   \hline  method 1   & method 2 &  & $\alpha(\bx')$ & $\beta(\bx')$\\
   \hline  \multirow{4}{*}{RKDE} & \multirow{4}{*}{KDE}
            & $R_1$      & 120   &  120  \\
            & & $R_2$    &   0   &   0 \\
            & & $T$      &  0    &    0  \\
            & & $p$-value& 0.00  &  0.00\\
   \hline
\end{tabular}
\end{center}
\caption{The signed-rank statistics and $p$-values of the Wilcoxon signed-rank test using the medians of $\{\alpha(\bx')\}$ and $\{\beta(\bx')\}$ as a performance measure. If $R_1$ is larger than $R_2$, method 1 is better than method 2.} \label{tab:rank_if}
\end{table}

\subsubsection{Kullback-Leibler (KL) divergence}
Second, we present the Kullback-Leibler (KL) divergence between density estimates $\widehat{f}$ and $f_0$,
\begin{equation*}
D_{KL}(\widehat{f} \,||\, f_0) = \int \widehat{f}(\bx) \log{\frac{\widehat{f}(\bx)}{f_0(\bx)}} d\bx.
\end{equation*}
This KL divergence is large whenever $\widehat{f}$ estimates $f_0$ to have mass where it does not.

The computation of $D_{KL}$ is done as follows. Since we do not know the nominal $f_0$, it is estimated as $\widetilde{f}_0$, a KDE based
on a separate nominal sample, obtained from the test data for each benchmark data set. Then, the integral is approximated by the sample mean, i.e.,
\begin{equation*}
D_{KL}(\widehat{f} \,||\, f_0) \approx \sum_{i=1}^{n'} \log{\frac{\widehat{f}(\bx'_i)}{\widetilde{f}_0(\bx'_i)}}
\end{equation*}
where $\{\bx'_i\}_{i=1}^{n'}$ is an i.i.d sample from the estimated density $\widehat{f}$ with $n' = 2 n = 2(n_0+n_1)$. Note that the estimated KL divergence can have an infinite value when $\widetilde{f}_0(\by) = 0$ (to machine precision) and $\widehat{f}(\by) > 0$ for some $\by \in \reals^d$. The averaged KL divergence over the permutations are used as the performance measure (smaller means better performance). Table \ref{tab:rank_kld} summarizes the results.

When comparing RKDEs and KDEs, the results show that KDEs have smaller KL divergence than RKDEs with $\epsilon =0$. As $\epsilon$ increases, however, RKDEs estimate $f_0$ more accurately than KDEs. The results also demonstrate that VKDEs are the worst in the sense of KL divergence. Note that VKDEs place a total mass of $1/n$ at all $\bX_i$, whereas the RKDE will place a mass $w_i < 1/n$ at outlying points.

\begin{table}[t]
\begin{center}
\begin{tabular}[c c c c c c c c c c c] {c| c| c| c c c c c c c c}
   \hline
   \hline  \multirow{2}{*}{method 1}   & \multirow{2}{*}{method 2} &  & \multicolumn{7}{c}{$\epsilon$}\\
   \cline{4-10}
              & &          & 0.00 & 0.05 & 0.10 & 0.15 & 0.20 & 0.25 & 0.30\\
   \hline  \multirow{4}{*}{RKDE} & \multirow{4}{*}{KDE}
            & $R_1$    & 26    &   67  &   78  &   83  &   94  &  101  &  103 \\
            & &$R_2$   & 94    &   53  &   42  &   37  &   26  &   19  &   17 \\
            & & $T$      & 26    &   53  &   42  &   37  &   26  &   19  &   17 \\
            & & $p$-value  & 0.06  &  0.72  &  0.33 & 0.21 & 0.06 & 0.02 & 0.01\\
   \hline  \multirow{4}{*}{RKDE} & \multirow{4}{*}{VKDE}
            & $R_1$    & 104   &  117  &  117  &  117  &  117  &  119  &  119 \\
            & & $R_2$   & 16    &   3   &    3  &    3  &    3  &    1  &    1 \\
            & & $T$      & 16    &   3   &    3  &    3  &    3  &    1  &    1 \\
            & & $p$-value  & 0.01  & 0.00  &  0.00 & 0.00  & 0.00  & 0.00 & 0.00\\
   \hline  \multirow{4}{*}{VKDE} & \multirow{4}{*}{KDE}
            & $R_1$    &  0    &    0  &    0  &    0  &    0  &    0  &    0 \\
            & & $R_2$  &  120  &  120  &  120  &  120  &  120  &  120  &  120 \\
            & & $T$      &  0    &    0  &    0  &    0  &    0  &    0  &    0 \\
            & & $p$-value  & 0.00  & 0.00  & 0.00  & 0.00  & 0.00 & 0.00 & 0.00\\
   \hline
\end{tabular}
\end{center}
\caption{The signed-rank statistics and $p$-values of the Wilcoxon signed-rank test using KL divergence as a performance measure. If $R_1$ is larger than $R_2$, method 1 is better than method 2.} \label{tab:rank_kld}
\end{table}

\subsubsection{Anomaly detection}
In this experiment, we apply the density estimators in anomaly detection problems. If we had a pure sample from $f_0$, we would estimate $f_0$ and use $\{\bx : \widehat{f}_0(\bx) > \lambda\}$ as a detector. For each $\lambda$, we could get a false negative and false positive probability using test data. By varying $\lambda$, we would then obtain a receiver operating characteristic (ROC) and area under the curve (AUC). However, since we have a contaminated sample, we have to estimate $f_0$ robustly. Robustness can be checked by comparing the AUC of the anomaly detectors, where the density estimates are based on the contaminated training data (higher AUC means better performance).

Examples of the ROCs are shown in Figure \ref{fig:ROC}. The RKDE provides better detection probabilities, especially at low false alarm rates. This results in higher AUC. For each pair of methods and each $\epsilon$, $R_1$, $R_2$, $T$ and $p$-values are shown in Table \ref{tab:rank_auc}. The results indicate that RKDEs are significantly better than KDEs when $\epsilon \geq 0.20$ with significance level $0.05$. RKDEs are also better than VKDEs when $\epsilon \geq 0.15$ but the difference is not significant. We also note that we have also evaluated the kernelized spatial depth (KSD) \citep{yixin09} in this setting. While this method does not yield a density estimate, it does aim to estimate density contours robustly. We found that the KSD performs worse in terms of AUC that either the RKDE or KDE, so those results are omitted \citep{kim11}.

\begin{figure}[!tb]
\centering
\begin{minipage}[htb]{.45\linewidth}
  \centering
  \includegraphics[width=1.0\linewidth]{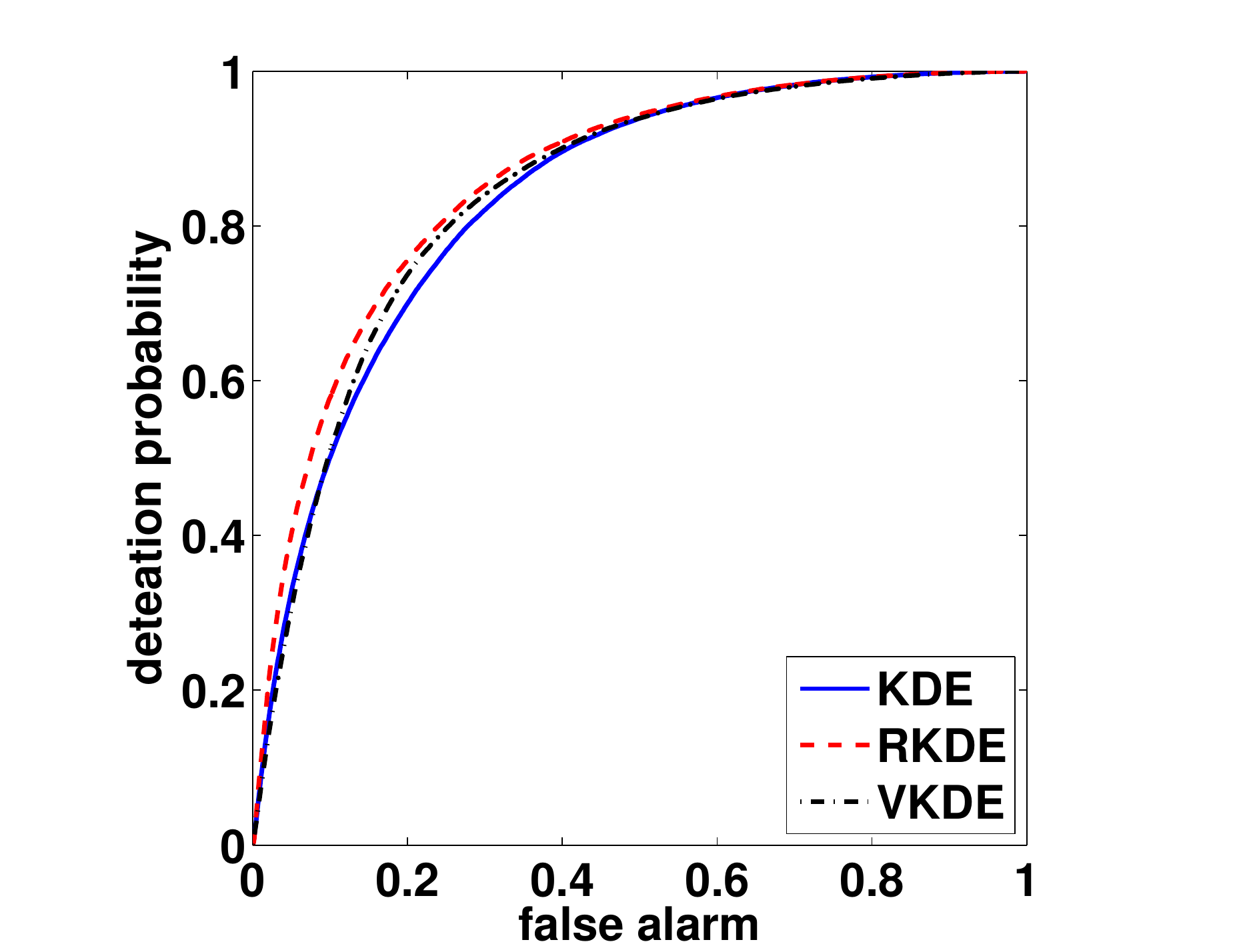}
  \centerline{(a) Banana, $\epsilon= 0.2$}\medskip
\end{minipage}
\hfill
\begin{minipage}[htb]{.45\linewidth}
  \centering
  \includegraphics[width=1.0\linewidth]{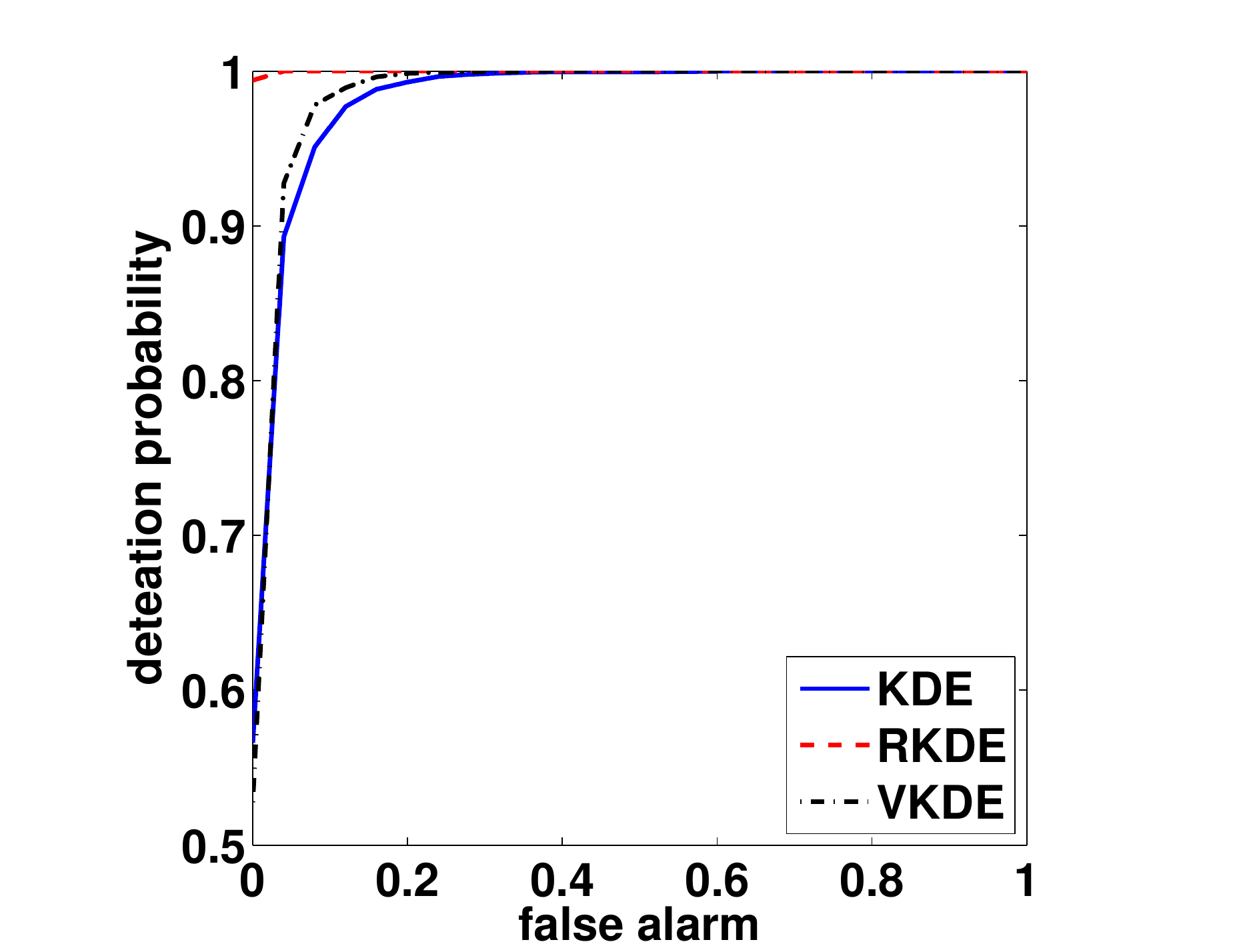}
  \centerline{(b) Iris, $\epsilon = 0.1$}\medskip
\end{minipage}
\caption{Examples of ROCs.} \label{fig:ROC}
\end{figure}

\begin{table}[t]
\begin{center}
\begin{tabular}[c c c c c c c c c c c] {c| c| c| c c c c c c c c}
   \hline
   \hline  \multirow{2}{*}{method 1}   & \multirow{2}{*}{method 2} &  & \multicolumn{7}{c}{$\epsilon$}\\
   \cline{4-10}
              & &          & 0.00 & 0.05 & 0.10 & 0.15 & 0.20 & 0.25 & 0.30\\
   \hline  \multirow{4}{*}{RKDE} & \multirow{4}{*}{KDE}
            & $R_1$      & 26    &   46  &   67  &   90  &   95  &   96  &   99 \\
            & & $R_2$    & 94    &   74  &   53  &   30  &   25  &   24  &   21 \\
            & & $T$      & 26    &   46  &   53  &   30  &   25  &   24  &   21 \\
            & & $p$-value  & 0.06  &  0.45  &  0.72 & 0.09 & 0.05 & 0.04 & 0.03\\
   \hline  \multirow{4}{*}{RKDE} & \multirow{4}{*}{VKDE}
            & $R_1$      & 33    &   49  &   58  &   75  &   80  &   90  &   86 \\
            & & $R_2$    & 87    &   71  &   62  &   45  &   40  &   30  &   34 \\
            & & $T$      & 33    &   49  &   58  &   45  &   40  &   30  &   34 \\
            & & $p$-value  & 0.14  & 0.56  &  0.93 & 0.42  & 0.28  & 0.09 & 0.15\\
   \hline  \multirow{4}{*}{VKDE} & \multirow{4}{*}{KDE}
            & $R_1$      & 38    &   70  &   79  &   91  &   95  &   96  &   99 \\
            & & $R_2$    & 82    &   50  &   41  &   29  &   25  &   24  &   21 \\
            & & $T$      & 38    &   50  &   41  &   29  &   25  &   24  &   21 \\
            & & $p$-value  & 0.23  & 0.60  & 0.30  & 0.08  & 0.05 & 0.04 & 0.03\\
   \hline
\end{tabular}
\end{center}
\caption{The signed-rank statistics of the Wilcoxon signed-rank test using AUC as a performance measure. If $R_1$ is larger than $R_2$, method 1 is better than method 2.} \label{tab:rank_auc}
\end{table}

\section{Conclusions}\label{sec:conclusion}
When kernel density estimators employ a smoothing kernel that is also a PSD kernel, they may be viewed as $M$-estimators in the RKHS associated with the kernel. While the traditional KDE corresponds to the quadratic loss, the RKDE employs a robust loss to achieve robustness to contamination of the training sample. The RKDE is a weighted kernel density estimate, where smaller weights are given to more outlying data points. These weights can be computed efficiently using a kernelized iteratively re-weighted least squares algorithm. The decreased sensitivity of RKDEs to contamination is further attested by the influence function, as well as experiments on anomaly detection and density estimation problems.

Robust kernel density estimators are nonparametric, making no parametric assumptions on the data generating distributions. However, their success is still contingent on certain conditions being satisfied. Obviously, the percentage of contaminating data must be less than $50\%$; our experiments examine contamination up to around $25\%$. In addition, the contaminating distribution must be outlying with respect to the nominal distribution.  Furthermore, the anomalous component should not be too concentrated, otherwise it may look like a mode of the nominal component. Such assumptions seem necessary given the unsupervised nature of the problem, and are implicit in our interpretation of the representer theorem and influence functions.

Although our focus has been on density estimation, in many applications the ultimate goal is not to estimate a density, but rather to estimate
decision regions. Our methodology is immediately applicable to such situations, as evidenced by our experiments on anomaly detection. It is
only necessary that the kernel be PSD here; the assumption that the kernel be nonnegative and integrate to one can clearly be dropped. This allows
for the use of more general kernels, such as polynomial kernels, or kernels on non-Euclidean domains such as strings and trees. The learning
problem here could be described as one-class classification with contaminated data.

In future work it would be interesting to investigate asymptotics, the bias-variance trade-off, and the efficiency-robustness trade-off of robust
kernel density estimators, as well as the impact of different losses and kernels.

\section{Proofs}\label{sec:proofs}
We begin with three lemmas and proofs. The first lemma will be used in the proofs of Lemma \ref{lem:lindep} and Theorem \ref{thm:influence_emp}, the
second one in the proof of Lemma \ref{lemma:strict_convex}, and the third one in the proof of Theorem \ref{thm:convergence_bm}.

\begin{lemma}\label{lem:lindep1}
Let $\bz_1, \dots, \bz_m$ be distinct points in $\reals^d$. If $K = (k(\bz_i, \bz_j))_{i, j= 1}^n$ is positive definite, then $\Phi(\bz_i) = k(\,\cdot\,, \bz_i)$'s are linearly independent.
\end{lemma}

\begin{proof}
$\sum_{i=1}^m \alpha_i \Phi(\bz_i)= 0$ implies
\begin{align*}
0 & = \biggl\|\sum_{i=1}^m \alpha_i \Phi(\bz_i)\biggr\|_\mathcal{H}^2\\
& = \biggl\langle \sum_{i=1}^m \alpha_i \Phi(\bz_i), \sum_{j=1}^m \alpha_j \Phi(\bz_j) \biggr\rangle_\mathcal{H}\\
& = \sum_{i=1}^m \sum_{j=1}^m \alpha_i\alpha_j k(\bz_i, \bz_j)
\end{align*}
and from positive definiteness of $K$, $\alpha_1 = \cdots = \alpha_m = 0$.
\end{proof}

\begin{lemma}\label{lem:lindep}
Let $\mathcal{H}$ be a RKHS associated with a kernel $k$, and $\bx_1$, $\bx_2$, and $\bx_3$ be distinct points in $\reals^d$. Assume that $K =(k(\bx_i, \bx_j))_{i,j=1}^3$ is positive definite. For any $g, h \in \mathcal{H}$ with $g \neq h$, $\Phi(\bx_i) -g$
and $\Phi(\bx_i) - h$ are linearly independent for some $i \in \{1, 2, 3\}$.
\end{lemma}

\begin{proof}
We will prove the lemma by contradiction. Suppose $\Phi(\bx_i) -g$ and $\Phi(\bx_i) - h$ are linearly dependent for all $i = 1, 2, 3$. Then,
there exists $(\alpha_i, \beta_i) \neq (0, 0)$ for $i = 1, 2, 3$ such that
\begin{align}
\label{eqn:ld1}\alpha_1(\Phi(\bx_1) - g) &+ \beta_1 (\Phi(\bx_1) - h) = \bzero\\
\label{eqn:ld2}\alpha_2(\Phi(\bx_2) - g) &+ \beta_2 (\Phi(\bx_2) - h) = \bzero\\
\label{eqn:ld3}\alpha_3(\Phi(\bx_3) - g) &+ \beta_3 (\Phi(\bx_3) - h) = \bzero.
\end{align}
Note that $\alpha_i  + \beta_i \neq 0$ since $g \neq h$.

First consider the case $\alpha_2 = 0$. This gives $h = \Phi(\bx_2)$, and $\alpha_1 \neq 0$ and $\alpha_3 \neq 0$. Then, (\ref{eqn:ld1})
and (\ref{eqn:ld2}) simplify to
\begin{align*}
g = \frac{\alpha_1+\beta_1}{\alpha_1} \Phi(\bx_1) - \frac{\beta_1}{\alpha_1}\Phi(\bx_2),\\
g = \frac{\alpha_3+\beta_3}{\alpha_3} \Phi(\bx_3) - \frac{\beta_3}{\alpha_3}\Phi(\bx_2),
\end{align*}
respectively. This is contradiction because $\Phi(\bx_1)$, $\Phi(\bx_2)$, and $\Phi(\bx_3)$ are linearly independent by Lemma \ref{lem:lindep1} and
\begin{equation*}
\frac{\alpha_1+\beta_1}{\alpha_1} \Phi(\bx_1) + \biggl(\frac{\beta_3}{\alpha_3}-\frac{\beta_1}{\alpha_1}\biggr)\Phi(\bx_2)
- \frac{\alpha_3+\beta_3}{\alpha_3} \Phi(\bx_3) = \bzero
\end{equation*}
where $(\alpha_1+\beta_1)/\alpha_1 \neq 0$.

Now consider the case where $\alpha_2 \neq 0$. Subtracting (\ref{eqn:ld2}) multiplied by $\alpha_1$ from (\ref{eqn:ld1}) multiplied
by $\alpha_2$ gives
\begin{equation*}
(\alpha_1 \beta_2 - \alpha_2 \beta_1) h = - \alpha_2(\alpha_1 + \beta_1) \Phi(\bx_1) + \alpha_1( \alpha_2 + \beta_2)\Phi(\bx_2).
\end{equation*}
In the above equation $\alpha_1 \beta_2 - \alpha_2 \beta_1 \neq 0$ because this implies $\alpha_2(\alpha_1 + \beta_1) = 0$ and
$\alpha_1( \alpha_2 + \beta_2) = 0$, which, in turn, implies $\alpha_2 = 0$. Therefore, $h$ can be expressed as
$h = \lambda_1 \Phi(\bx_1) + \lambda_2 \Phi(\bx_2)$ where
\begin{align*}
\lambda_1 = -\frac{\alpha_2(\alpha_1+\beta_1)}{\alpha_1\beta_2 -\alpha_2\beta_1},\quad \lambda_2
= \frac{\alpha_1(\alpha_2+\beta_2)}{\alpha_1\beta_2 -\alpha_2\beta_1}.
\end{align*}
Similarly, from (\ref{eqn:ld2}) and (\ref{eqn:ld3}), $h = \lambda_3 \Phi(\bx_2) + \lambda_4 \Phi(\bx_3)$ where
\begin{align*}
\lambda_3 = -\frac{\alpha_3(\alpha_2+\beta_2)}{\alpha_2\beta_3 -\alpha_3\beta_2}, \quad \lambda_4
= \frac{\alpha_2(\alpha_3+\beta_3)}{\alpha_2\beta_3 -\alpha_3\beta_2}.
\end{align*}
Therefore, we have $h = \lambda_1 \Phi(\bx_1) + \lambda_2 \Phi(\bx_2) = \lambda_3 \Phi(\bx_2) + \lambda_4 \Phi(\bx_3)$.
Again, from the linear independence of $\Phi(\bx_1)$, $\Phi(\bx_2)$, and $\Phi(\bx_3)$, we have $\lambda_1 = 0$,
$\lambda_2 = \lambda_3$, $\lambda_4 = 0$. However, $\lambda_1 =0$ leads to $\alpha_2 = 0$.

Therefore, $\Phi(\bx_i) -g$ and $\Phi(\bx_i) - h$ are linearly independent for some $i \in \{1, 2, 3\}$.
\end{proof}

\begin{lemma}\label{lemma:lemma_compact1}
Given $\bX_1, \dots, \bX_n$, let $\mathcal{D}_n \subset \mathcal{H}$ be defined as
\begin{equation*}
\mathcal{D}_n = \biggl\{g \, \bigg| \, g  = \sum_{i=1}^n w_i \cdot\Phi(\bX_i), \quad w_i \geq 0, \quad\sum_{i=1}^n w_i =1\biggr\}
\end{equation*}
Then, $\mathcal{D}_n$ is compact.
\end{lemma}

\begin{proof}
Define
\begin{equation*}
A = \biggl\{(w_1, \dots, w_n) \in \reals^n \biggl| \, w_i \geq 0, \quad \sum_{i=1}^n w_i = 1\biggr\},
\end{equation*}
and a mapping $W$
\begin{equation*}
W : (w_1, \dots, w_n) \in A \to \sum_{i=1}^n w_i \cdot\Phi(\bX_i) \in \mathcal{H}.
\end{equation*}
Note that $A$ is compact, $W$ is continuous, and $\mathcal{D}_n$ is the image of $A$ under $W$. Since the continuous image of
a compact space is also compact \citep{munkres00}, $\mathcal{D}_n$ is compact.
\end{proof}

\subsection{Proof of Lemma \ref{lemma:v_mu}}\label{ssec:proof_lemma_v_mu}
We begin by calculating the Gateaux differential of $J$. We consider the two cases: $\Phi(\bx) -(g + \alpha h) = \bzero$ and $\Phi(\bx) -(g + \alpha h) \neq \bzero$.

For $\Phi(\bx) -(g + \alpha h) \neq \bzero$,
\begin{eqnarray}\label{eqn:deriv_rho1}
\nonumber \lefteqn{\frac{\partial}{\partial \alpha} \rho\bigl(\|\Phi(\bx) -(g + \alpha h)\|_{\mathcal{H}}\bigr)}\\
\nonumber & = & \psi\bigl(\|\Phi(\bx) -(g + \alpha h)\|_{\mathcal{H}}\bigr)\cdot\frac{\partial}{\partial \alpha}
\|\Phi(\bx) -(g + \alpha h)\|_{\mathcal{H}}\\
\nonumber & = & \psi\bigl(\|\Phi(\bx) -(g + \alpha h)\|_{\mathcal{H}}\bigr)\cdot\frac{\partial}{\partial \alpha}
\sqrt{\|\Phi(\bx) -(g + \alpha h)\|_{\mathcal{H}}^2}\\
\nonumber & = & \psi\bigl(\|\Phi(\bx) -(g + \alpha h)\|_{\mathcal{H}}\bigr)\cdot\frac{\frac{\partial}{\partial \alpha}
\|\Phi(\bx) -(g + \alpha h)\|_{\mathcal{H}}^2}{2\sqrt{\|\Phi(\bx) -(g + \alpha h)\|_{\mathcal{H}}^2}} \\
\nonumber & = & \frac{\psi\bigl(\|\Phi(\bx) -(g + \alpha h)\|_{\mathcal{H}}\bigr)}{2\|\Phi(\bx) -(g + \alpha h)\|_{\mathcal{H}}}
\cdot\frac{\partial}{\partial \alpha} \biggl(\|\Phi(\bx) - g\|_{\mathcal{H}}^2 -2 \bigl\langle \Phi(\bx) - g, \alpha h
\bigr\rangle_\mathcal{H} + \alpha^2 \|h\|_{\mathcal{H}}^2\biggr) \\
\nonumber & = & \frac{\psi\bigl(\|\Phi(\bx) -(g + \alpha h)\|_{\mathcal{H}}\bigr)}{\|\Phi(\bx) -(g + \alpha h)\|_{\mathcal{H}}}
\cdot\biggl( - \bigl\langle \Phi(\bx) - g, h \bigr\rangle_\mathcal{H} + \alpha \|h\|_{\mathcal{H}}^2\biggr)\\
& = & \varphi\bigl(\|\Phi(\bx) -(g + \alpha h)\|_{\mathcal{H}}\bigr)\cdot\bigl( - \bigl\langle \Phi(\bx) - (g+\alpha h), h \bigr\rangle_\mathcal{H} \bigr).
\end{eqnarray}
For $\Phi(\bx) -(g + \alpha h) = \bzero$,
\begin{eqnarray}\label{eqn:deriv_rho2}
\nonumber \lefteqn{\frac{\partial}{\partial \alpha} \rho\bigl(\|\Phi(\bx) -(g + \alpha h)\|_{\mathcal{H}}\bigr)}\\
\nonumber & = & \lim_{\delta \to 0}\frac{\rho\bigl(\|\Phi(\bx) -(g + (\alpha +\delta) h)\|_{\mathcal{H}}\bigr) -
\rho\bigl(\|\Phi(\bx) -(g + \alpha h)\|_{\mathcal{H}}\bigr)}{\delta}\\
\nonumber & = & \lim_{\delta \to 0}\frac{\rho\bigl(\|\delta h\|_{\mathcal{H}}\bigr) - \rho\bigl(0\bigr)}{\delta}\\
\nonumber & = & \lim_{\delta \to 0}\frac{\rho\bigl(\delta \|h\|_{\mathcal{H}}\bigr)}{\delta}\\
\nonumber & = &
\begin{cases}
\lim_{\delta \to 0}\frac{\rho(0)}{\delta}, \quad & h = \bzero\\
\lim_{\delta \to 0}\frac{\rho(\delta \|h\|_{\mathcal{H}})}{\delta \|h\|_{\mathcal{H}}}\cdot \|h\|_{\mathcal{H}}, \quad & h \neq \bzero
\end{cases}\\
\nonumber & = & 0 \\
& = & \varphi\bigl(\|\Phi(\bx) -(g + \alpha h)\|_{\mathcal{H}}\bigr)\cdot\bigl( - \bigl\langle \Phi(\bx) - (g+\alpha h), h \bigr\rangle_\mathcal{H} \bigr)
\end{eqnarray}
where the second to the last equality comes from (A1) and the last equality comes from the facts that $\Phi(\bx) - (g+\alpha h) = \bzero$ and $\varphi(0)$ is well-defined by (A2).

From (\ref{eqn:deriv_rho1}) and (\ref{eqn:deriv_rho2}), we can conclude that for any $g$, $h \in \mathcal{H}$, and $\bx \in \reals^d$,
\begin{eqnarray}\label{enq:deriv_rho}
\nonumber \lefteqn{\frac{\partial}{\partial \alpha} \rho\bigl(\|\Phi(\bx) -(g + \alpha h)\|_{\mathcal{H}}\bigr)}\\
& = &\varphi\bigl(\|\Phi(\bx) -(g + \alpha h)\|_{\mathcal{H}}\bigr)\cdot \bigl( - \bigl\langle \Phi(\bx) - (g + \alpha h), h \bigr\rangle_\mathcal{H} \bigr)
\end{eqnarray}
Therefore,
\begin{eqnarray}
\nonumber \lefteqn{\delta J(g; h) = \frac{\partial}{\partial \alpha} J(g+ \alpha h) \bigl |_{\alpha = 0}}\\
\nonumber & = &\frac{\partial}{\partial \alpha} \biggl(\frac{1}{n} \sum_{i=1}^n \rho\bigl(\|\Phi(\bX_i) -(g + \alpha h)\|_\mathcal{H}\bigr)\biggr) \biggl |_{\alpha = 0}\\
\nonumber & = &\frac{1}{n} \sum_{i=1}^n \frac{\partial}{\partial \alpha} \rho\bigl(\|\Phi(\bX_i) -(g + \alpha h)\|_\mathcal{H}\bigr) \biggl |_{\alpha = 0}\\
\nonumber & = &\frac{1}{n} \sum_{i=1}^n \varphi\bigl(\|\Phi(\bX_i) -(g + \alpha h)\|_\mathcal{H}\bigr)
\cdot\bigl( - \bigl \langle \Phi(\bX_i) - (g+\alpha h), h \bigr\rangle_\mathcal{H} \bigr) \biggl |_{\alpha = 0}\\
\nonumber & = & - \frac{1}{n} \sum_{i=1}^n \varphi\bigl(\|\Phi(\bX_i) - g\|_\mathcal{H}\bigr)\cdot
\bigl \langle \Phi(\bX_i) - g, h \bigr\rangle_\mathcal{H}\\
\nonumber & = &- \biggl \langle  \frac{1}{n} \sum_{i=1}^n \varphi\bigl(\|\Phi(\bX_i) - g\|_\mathcal{H}\bigr)
\cdot \bigl(\Phi(\bX_i) - g\bigr) , h \biggr\rangle_\mathcal{H}\\
\nonumber & = &- \bigl \langle V(g) , h \bigr\rangle_\mathcal{H}.
\end{eqnarray}

The necessary condition for $g$ to be a minimizer of $J$, i.e.,  $g = \widehat{f}_{RKDE}$, is that $\delta J(g; h) = 0,
\quad \forall h \in \mathcal{H}$, which leads to $V(g) = \bzero$.

\subsection{Proof of Theorem \ref{thm:representer}}\label{ssec:proof_thm_representer}
From Lemma \ref{lemma:v_mu}, $V(\widehat{f}_{RKDE}) = \bzero$, that is,
$$
\frac{1}{n}\sum_{i=1}^n \varphi(\|\Phi(\bX_i) -
\widehat{f}_{RKDE}\|_{\mathcal{H}})\cdot(\Phi(\bX_i) - \widehat{f}_{RKDE})
= \bzero.
$$
Solving for $\widehat{f}_{RKDE}$,
we have $\widehat{f}_{RKDE} = \sum_{i=1}^n w_i \Phi(\bX_i)$ where
\begin{equation*}
w_i = \biggl(\sum_{j=1}^n \varphi(\|\Phi(\bX_j) - \widehat{f}_{RKDE}\|_{\mathcal{H}})\biggr)^{-1}
\cdot\varphi(\|\Phi(\bX_i) - \widehat{f}_{RKDE}\|_{\mathcal{H}}).
\end{equation*}
Since $\rho$ is non-decreasing, $w_i \geq 0$. Clearly $\sum_{i=1}^n w_i =
1$


\subsection{Proof of Lemma \ref{lemma:strict_convex}}\label{ssec:strict_convex}
$J$ is strictly convex on $\mathcal{H}$ if for any $0 < \lambda <1 $, and $g, h \in \mathcal{H}$ with $g \neq h$
\begin{equation*}
J(\lambda g + (1-\lambda) h) < \lambda J(g) + (1-\lambda) J(h).
\end{equation*}
Note that
\begin{align*}
J(\lambda g + (1-\lambda) h) & = \frac{1}{n} \sum_{i=1}^n \rho\bigl(\|\Phi(\bX_i) - \lambda g - (1-\lambda) h\|_{\mathcal{H}}\bigr)\\
& = \frac{1}{n} \sum_{i=1}^n \rho\bigl(\|\lambda(\Phi(\bX_i) -  g) + (1-\lambda) (\Phi(\bX_i)- h)\|_{\mathcal{H}}\bigr)\\
& \leq \frac{1}{n} \sum_{i=1}^n \rho\bigl(\lambda\|\Phi(\bX_i) -  g\|_{\mathcal{H}} + (1-\lambda)\|\Phi(\bX_i)- h\|_{\mathcal{H}}\bigr)\\
& \leq \frac{1}{n} \sum_{i=1}^n \lambda \rho\bigl(\|\Phi(\bX_i) -  g\|_{\mathcal{H}}\bigr) + (1-\lambda)\rho\bigl(\|\Phi(\bX_i)- h\|_{\mathcal{H}}\bigr)\\
& = \lambda J(g)+ (1-\lambda) J(h).
\end{align*}
The first inequality comes from the fact that $\rho$ is non-decreasing and
\begin{equation*}
\|\lambda(\Phi(\bX_i) -  g) + (1-\lambda) (\Phi(\bX_i)- h)\|_{\mathcal{H}} \leq \lambda\|\Phi(\bX_i) -  g\|_{\mathcal{H}} +
(1-\lambda)\|\Phi(\bX_i)- h\|_{\mathcal{H}},
\end{equation*}
and the second inequality comes from the convexity of $\rho$.

Under condition \emph{(i)}, $\rho$ is strictly convex and thus the second inequality is strict, implying $J$ is strictly convex. Under condition \emph{(ii)}, we will show that
the first inequality is strict using proof by contradiction. Suppose the first inequality holds with equality. Since $\rho$ is strictly
increasing, this can happen only if
\begin{equation*}\label{eqn:equal}
\|\lambda(\Phi(\bX_i) -  g) + (1-\lambda) (\Phi(\bX_i)- h)\|_{\mathcal{H}} = \lambda\|\Phi(\bX_i) -  g\|_{\mathcal{H}} +
(1-\lambda)\|\Phi(\bX_i)- h\|_{\mathcal{H}},
\end{equation*}
for $i = 1, \dots, n$. Equivalently, it can happen only if $(\Phi(\bX_i) -  g)$ and $(\Phi(\bX_j) -  h)$
are linearly dependent for all $i = 1, \dots, n$. However, from $n \geq 3$ and positive definiteness of $K$, there exist three distinct $\bX_i$'s, say $\bZ_1$, $\bZ_2$, and $\bZ_3$ with positive definite $K' = (k_\sigma(\bZ_i, \bZ_j))_{i,j=1}^3$. By Lemma \ref{lem:lindep}, it must be the case that for some $i \in \{1, 2, 3\}$, $(\Phi(\bZ_i) -  g)$ and $(\Phi(\bZ_i) -  h)$ are linearly independent. Therefore, the inequality is strict, and thus
$J$ is strictly convex.

\subsection{Proof of Theorem \ref{thm:convergence_bm}}\label{ssec:proof_thm_convergence_bm}
First, we will prove the monotone decreasing property of $J(f^{(k)})$. Given $r \in \reals$, define
\begin{align*}
u(x; r) = \rho(r) - \frac{1}{2} r \psi(r) + \frac{1}{2}\varphi(r) x^2.
\end{align*}
If $\varphi$ is nonincreasing, then $u$ is a surrogate function of $\rho$, having the following property \citep{huber81}:
\begin{align}
\label{eqn:surrogate1} u(r; r) & = \rho(r)\\
\label{eqn:surrogate2} u(x; r) &\geq \rho(x), \quad \forall x.
\end{align}
Define
\begin{equation*}
Q(g; f^{(k)}) = \frac{1}{n}\sum_{i=1}^n u\bigl(\|\Phi(\bX_i) - g\|_\mathcal{H}, \|\Phi(\bX_i) - f^{(k)}\|_\mathcal{H}\bigr).
\end{equation*}
Note that since $\psi$ and $\varphi$ are continuous, $Q(\,\cdot\,;\,\cdot\,)$ is continuous in both arguments.

From (\ref{eqn:surrogate1}) and (\ref{eqn:surrogate2}), we have
\begin{align}
\nonumber Q(f^{(k)}; f^{(k)}) & = \frac{1}{n}\sum_{i=1}^n u\bigl(\|\Phi(\bX_i) - f^{(k)}\|_\mathcal{H}, \|\Phi(\bX_i) - f^{(k)}\|_\mathcal{H}\bigr)\\
\nonumber & = \frac{1}{n}\sum_{i=1}^n \rho(\|\Phi(\bX_i) - f^{(k)}\|_\mathcal{H})\\
\label{eqn:Q1}& = J(f^{(k)})
\end{align}
and
\begin{align}
\nonumber Q(g; f^{(k)}) & = \frac{1}{n}\sum_{i=1}^n u\bigl(\|\Phi(\bX_i) - g\|_\mathcal{H}, \|\Phi(\bX_i) - f^{(k)}\|_\mathcal{H}\bigr)\\
\nonumber & \geq \frac{1}{n}\sum_{i=1}^n \rho\bigl(\|\Phi(\bX_i) - g\|_\mathcal{H})\\
\label{eqn:Q2}& = J(g), \quad \forall g \in \mathcal{H}
\end{align}
The next iterate $f^{(k+1)}$ is the minimizer of $Q(g; f^{(k)})$ since
\begin{align}
\nonumber f^{(k+1)} & = \sum_{i=1}^n w_i^{(k)}\Phi(\bX_i)\\
\nonumber & = \sum_{i=1}^n \frac{\varphi(\|\Phi(\bX_i) - f^{(k)}\|_\mathcal{H})}{\sum_{j=1}^n\varphi(\|\Phi(\bX_j) - f^{(k)}\|_\mathcal{H})}\Phi(\bX_i)\\
\nonumber & = \argmin_{g \in \mathcal{H}} \sum_{i=1}^n \varphi(\|\Phi(\bX_i) - f^{(k)}\|_\mathcal{H}) \cdot \|\Phi(\bX_i) -g \|_\mathcal{H}^2\\
\label{eqn:Q3} & = \argmin_{g \in \mathcal{H}} Q(g; f^{(k)})
\end{align}
From (\ref{eqn:Q1}), (\ref{eqn:Q2}), and (\ref{eqn:Q3}),
\begin{equation*}
J(f^{(k)}) = Q(f^{(k)}; f^{(k)}) \geq Q(f^{(k+1)}; f^{(k)}) \geq J(f^{(k+1)})
\end{equation*}
and thus $J(f^{(k)})$ monotonically decreases at every iteration. Since $\{J(f^{(k)})\}_{k=1}^\infty$ is bounded below by $0$, it converges.

Next, we will prove that every limit point $f^*$ of $\{f^{(k)}\}_{k=1}^\infty$ belongs to $\mathcal{S}$. Since the sequence
$\{f^{(k)}\}_{k=1}^\infty$ lies in the compact set $\mathcal{D}_n$ (see Theorem \ref{thm:representer} and Lemma \ref{lemma:lemma_compact1}), it has a convergent subsequence $\{f^{(k_l)}\}_{l=1}^\infty$.
Let $f^*$ be the limit of $\{f^{(k_l)}\}_{l=1}^\infty$. Again, from (\ref{eqn:Q1}), (\ref{eqn:Q2}), and (\ref{eqn:Q3}),
\begin{align*}
Q(f^{(k_{l+1})} ; f^{(k_{l+1})} )& = J(f^{(k_{l+1})})\\
& \leq J(f^{(k_{l}+1)})\\
& \leq Q(f^{(k_{l}+1)} ; f^{(k_{l})})\\
& \leq Q(g; f^{(k_{l})}) \quad, \forall g \in \mathcal{H},
\end{align*}
where the first inequality comes from the monotone decreasing property of $J(f^{(k)})$.
By taking the limit on the both side of the above inequality, we have
\begin{align*}
Q(f^* ; f^*) &\leq Q(g; f^*) \quad, \forall g \in \mathcal{H}.
\end{align*}
Therefore,
\begin{align*}
f^* & = \argmin_{g \in \mathcal{H}} Q(g; f^*)\\
& = \sum_{i=1}^n \frac{\varphi(\|\Phi(\bX_i) - f^*\|_\mathcal{H})}{\sum_{j=1}^n\varphi(\|\Phi(\bX_j) - f^*\|_\mathcal{H})}\Phi(\bX_i)
\end{align*}
and thus
\begin{align*}
\sum_{i=1}^n \varphi(\|\Phi(\bX_i) - f^*\|_\mathcal{H})\cdot(\Phi(\bX_i)-f^*) = \bzero.
\end{align*}
This implies $f^* \in \mathcal{S}$.

Now we will prove $\|f^{(k)} - \mathcal{S}\|_\mathcal{H} \to 0$ by contradiction. Suppose $\inf_{g \in \mathcal{S}} \|f^{(k)} - g\|_\mathcal{H}
\nrightarrow 0$. Then, there exists $\epsilon > 0$ such that $\forall K \in \mathbb{N}$, $\exists k > K$ with $\inf_{g \in \mathcal{S}}
\|f^{(k)} - g\|_\mathcal{H} \geq \epsilon$. Thus, we can construct an increasing sequence of indices $\{k_l\}_{l=1}^\infty$ such that
$\inf_{g \in \mathcal{S}} \|f^{(k_l)} - g\|_\mathcal{H} \geq \epsilon$ for all $l= 1, 2, \dots$. Since $\{f^{(k_l)}\}_{l=1}^\infty$ lies
in the compact set $\mathcal{D}_n$, it has a subsequence converging to some $f^\dag$, and we can choose $j$ such that
$\|f^{(k_j)} - f^\dag\|_\mathcal{H} < \epsilon/2$. Since $f^\dag$ is also a limit point of $\{f^{(k)}\}_{k=1}^\infty$, $f^\dag \in \mathcal{S}$.
This is a contradiction because
\begin{align*}
\epsilon \leq \inf_{g \in \mathcal{S}} \|f^{(k_j)} - g\|_\mathcal{H} \leq \|f^{(k_j)} - f^\dag\|_\mathcal{H} \leq \epsilon/2.
\end{align*}

\subsection{Proof of Theorem \ref{thm:influence_true}}\label{ssec:proof_thm_influence_true}
Since the RKDE is given as $\widehat{f}_{RKDE}(\bx; F) = \langle \Phi(\bx), f_{F} \rangle_\mathcal{H}$, the influence function for the RKDE is
\begin{align*}
IF(\bx, \bx'; \widehat{f}_{RKDE}, F) &= \lim_{s\to 0} \frac{\widehat{f}_{RKDE}(\bx; F_s) - \widehat{f}_{RKDE}(\bx; F)}{s}\\
& = \lim_{s\to 0} \frac{\langle \Phi(\bx), f_{F_s} \rangle_\mathcal{H} - \langle \Phi(\bx), f_{F} \rangle_\mathcal{H}}{s}\\
& = \biggl \langle \Phi(\bx), \lim_{s \to 0} \frac{f_{F_s} - f_{F}}{s} \biggr \rangle_\mathcal{H}
\end{align*}
and thus we need to find $\dot{f}_{F} \triangleq \lim_{s \to 0} \frac{f_{F_s} - f_{F}}{s}$.

As we generalize the definition of RKDE from $\widehat{f}_{RKDE}$ to $f_{F}$, the necessary condition $V(\widehat{f}_{RKDE})$ also generalizes. However, a few things must be taken care of since we are dealing with integral instead of summation. Suppose $\psi$ and $\varphi$ are bounded by $B'$ and $B''$, respectively. Given a probability measure $\mu$, define
\begin{equation}\label{eqn:extended_J}
J_{\mu}(g) = \int \rho(\|\Phi(\bx) - g\|_\mathcal{H}) \, d\mu(\bx).
\end{equation}
From (\ref{enq:deriv_rho}),
\begin{eqnarray}\label{eqn:delta_J}
\nonumber \lefteqn{\delta J_{\mu}(g; h) = \frac{\partial}{\partial \alpha} J_{\mu}(g+ \alpha h) \bigl |_{\alpha = 0}}\\
\nonumber & = &\frac{\partial}{\partial \alpha} \int \rho\bigl(\|\Phi(\bx) -(g + \alpha h)\|_\mathcal{H}\bigr)\, d\mu(\bx) \biggl |_{\alpha = 0}\\
\nonumber & = &\int \frac{\partial}{\partial \alpha} \rho\bigl(\|\Phi(\bx) -(g + \alpha h)\|_\mathcal{H}\bigr)\, d\mu(\bx) \biggl |_{\alpha = 0}\\
\nonumber & = &\int \varphi\bigl(\|\Phi(\bx) -(g + \alpha h)\|_\mathcal{H}\bigr)
\cdot\bigl( - \bigl \langle \Phi(\bx) - (g+\alpha h), h \bigr\rangle_\mathcal{H} \bigr) \, d\mu(\bx) \biggl |_{\alpha = 0}\\
\nonumber & = & - \int \varphi\bigl(\|\Phi(\bx) - g\|_\mathcal{H}\bigr)\cdot
\bigl \langle \Phi(\bx) - g, h \bigr\rangle_\mathcal{H} \, d\mu(\bx)\\
\nonumber & = &- \int \biggl \langle  \varphi\bigl(\|\Phi(\bx) - g\|_\mathcal{H}\bigr)
\cdot \bigl(\Phi(\bx) - g\bigr) , h \biggr\rangle_\mathcal{H} \, d\mu(\bx).
\end{eqnarray}
The exchange of differential and integral is valid \citep{serge93} since for any fixed $g, h \in \mathcal{H}$, and $\alpha \in (-1, 1)$
\begin{eqnarray*}
\lefteqn{\biggl|\frac{\partial}{\partial \alpha} \rho\bigl(\|\Phi(\bx) -(g + \alpha h)\|_\mathcal{H}\bigr)\biggr|}\\
& = & \varphi\bigl(\|\Phi(\bx) -(g + \alpha h)\|\bigr)\cdot
\bigl| - \bigl\langle \Phi(\bx) - (g+\alpha h), h \bigr\rangle_\mathcal{H}\bigr|\\
& \leq & B''\cdot \|\Phi(\bx) - (g+\alpha h)\|\cdot \|h\|_\mathcal{H}\\
& \leq & B''\cdot \bigl(\|\Phi(\bx)\|_\mathcal{H} + \|g\|_\mathcal{H} + \|h\|_\mathcal{H}\bigr) \cdot \|h\|_\mathcal{H}\\
& \leq & B''\cdot \bigl(\tau + \|g\|_\mathcal{H} + \|h\|_\mathcal{H}\bigr)\cdot \|h\|_\mathcal{H} < \infty.
\end{eqnarray*}

Since $\varphi(\|\Phi(\bx) - g\|_\mathcal{H})\cdot \bigl(\Phi(\bx) - g\bigr)$ is strongly integrable, i.e.,
\begin{equation*}
\int \bigl \| \varphi\bigl(\|\Phi(\bx) - g\|_\mathcal{H}\bigr)\cdot
\bigl(\Phi(\bx) - g\bigr)\bigr\|_\mathcal{H} \, d\mu(\bx) \leq B' < \infty,
\end{equation*}
its Bochner-integral \citep{berlinet04}
\begin{equation*}
V_{\mu} (g) \triangleq \int \varphi(\|\Phi(\bx) - g\|_{\mathcal{H}})\cdot(\Phi(\bx) - g) \, d\mu(\bx)
\end{equation*}
is well-defined. Therefore, we have
\begin{align*}
\delta J_{\mu}(g; h) & = - \biggl \langle \int \varphi\bigl(\|\Phi(\bx) - g\|_{\mathcal{H}}\bigr)
\cdot \bigl(\Phi(\bx) - g\bigr) \, d\mu(\bx) , h \biggr\rangle_\mathcal{H}\\
& = - \bigl \langle V_{\mu}(g) , h \bigr\rangle_\mathcal{H}.
\end{align*}
and $V_{\mu}(f_{\mu}) = \bzero$.


From the above condition for $f_{F_s}$, we have
\begin{align*}
\bzero & = V_{F_s}(f_{F_s})\\
& = (1-s)\cdot V_F(f_{F_s}) + s V_{\delta_{\bx'}}(f_{F_s}), \quad \forall s \in [0, 1)
\end{align*}
Therefore,
\begin{align*}
\bzero & = \lim_{s\to 0} (1-s)\cdot V_F(f_{F_s}) + \lim_{s \to 0} s \cdot V_{\delta_{\bx'}}(f_{F_s})\\
& = \lim_{s\to 0} V_F(f_{F_s}).
\end{align*}
Then,
\begin{align}\label{eqn:icderiv1}
\nonumber \bzero = & \lim_{s \to 0} \frac{1}{s} \biggl(V_{F_s}(f_{F_s}) - V_F(f_{F}) \biggr)\\
\nonumber = & \lim_{s \to 0} \frac{1}{s} \biggl((1-s)V_F(f_{F_s}) + sV_{\delta_{\bx'}}(f_{F_s}) - V_F(f_{F}) \biggr)\\
\nonumber =& \lim_{s \to 0} \frac{1}{s} \biggl(V_F(f_{F_s}) - V_F(f_{F})\biggr) - \lim_{s\to 0} V_F(f_{F_s})
+ \lim_{s\to 0} V_{\delta_{\bx'}}(f_{F_s})\\
\nonumber =& \lim_{s \to 0} \frac{1}{s} \biggl(V_F(f_{F_s}) - V_F(f_{F})\biggr) + \lim_{s\to 0} V_{\delta_{\bx'}}(f_{F_s})\\
\nonumber = & \lim_{s \to 0} \frac{1}{s} \biggl(V_F(f_{F_s}) - V_F(f_{F})\biggr) + \lim_{s\to 0} \varphi(\|\Phi(\bx') - f_{F_s}\|)\cdot(\Phi(\bx')-f_{F_s})\\
= & \lim_{s \to 0} \frac{1}{s} \biggl(V_F(f_{F_s}) - V_F(f_{F})\biggr) + \varphi(\|\Phi(\bx') - f_{F}\|)\cdot(\Phi(\bx')-f_{F}).
\end{align}
where the last equality comes from the facts that $f_{F_s} \to f_{F}$ and continuity of $\varphi$.

Let $U$ denote the mapping $\mu \mapsto f_{\mu}$. Then,
\begin{align}\label{eqn:icderiv4}
\nonumber \dot{f}_{F} & \triangleq \lim_{s \to 0} \frac{f_{F_s} - f_{F}}{s}\\
\nonumber & = \lim_{s \to 0} \frac{U(F_s) - U(F)}{s}\\
\nonumber & = \lim_{s \to 0} \frac{U\bigl((1-s)F + s \delta_{\bx'}\bigr) - U(F)}{s}\\
\nonumber & = \lim_{s \to 0} \frac{U\bigl(F + s (\delta_{\bx'}-F)\bigr) - U(F)}{s}\\
& = \delta U(F; \delta_{\bx'} - F)
\end{align}
where $\delta U(P; Q)$ is the Gateaux differential of $U$ at $P$ with increment $Q$. The first term in (\ref{eqn:icderiv1}) is
\begin{eqnarray}\label{eqn:icderiv2}
\nonumber\lefteqn{\lim_{s \to 0} \frac{1}{s} \biggl(V_F\bigl(f_{F_s}\bigr) - V_F\bigl(f_{F}\bigr)\biggr)}\\
\nonumber& = & \lim_{s \to 0} \frac{1}{s} \biggl(V_F\bigl(U(F_s)\bigr) - V_F\bigl(U(F)\bigr)\biggr)\\
\nonumber& = & \lim_{s \to 0} \frac{1}{s} \biggl((V_F\circ U) \bigl(F_s) - (V_F\circ U)(F)\biggr)\\
\nonumber& = & \lim_{s \to 0} \frac{1}{s} \biggl((V_F\circ U) \bigl(F+s(\delta_{\bx'} - F)\bigr) - (V_F\circ U)(F)\biggr)\\
\nonumber& = & \delta (V_F\circ U) (F; \delta_{\bx'} - F)\\
\nonumber& = & \delta V_F\bigl(U(F); \delta U(F; \delta_{\bx'} - F)\bigr)\\
& = & \delta V_F\bigl(f_{F}; \dot{f}_{F}\bigr)
\end{eqnarray}
where we apply the chain rule of Gateaux differential, $\delta (G\circ H) (u; x) = \delta G(H(u); \delta H(u; x))$, in the second to the last equality. Although $\dot{f}_{F}$ is technically not a Gateaux differential since the space of probability distributions is not a vector space, the chain rule still applies.

Thus, we only need to find the Gateaux differential of $V_F$. For $g, h \in \mathcal{H}$
\begin{eqnarray}\label{eqn:icderiv3}
\nonumber \lefteqn{\delta V_F (g; h) = \lim_{s\to 0} \frac{1}{s} \biggl(V_F(g + s\cdot h) - V_F(g)\biggr)}\\
\nonumber& = & \lim_{s\to 0} \frac{1}{s} \biggl(\int \varphi(\|\Phi(\bx) - g - s\cdot h\|_\mathcal{H})
\cdot (\Phi(\bx) - g - s\cdot h) dF(\bx)\\
\nonumber& & \qquad \qquad- \int \varphi(\|\Phi(\bx) - g\|_\mathcal{H}) \cdot (\Phi(\bx) - g) dF(\bx)\biggr)\\
\nonumber& = & \lim_{s\to 0} \frac{1}{s} \int \biggl(\varphi(\|\Phi(\bx) - g - s\cdot h\|_\mathcal{H})
 - \varphi(\|\Phi(\bx) -g\|_\mathcal{H})\biggr)\cdot (\Phi(\bx) -g) dF(\bx)\\
\nonumber& &  - \lim_{s\to 0} \frac{1}{s} \int \biggl(\varphi(\|\Phi(\bx) - g - s\cdot h\|_\mathcal{H})
\cdot s\cdot h\biggr)\, dF(\bx)\\
\nonumber& = & \int \lim_{s\to 0} \frac{1}{s} \biggl(\varphi(\|\Phi(\bx) - g - s\cdot h\|_\mathcal{H})
- \varphi(\|\Phi(\bx) -g\|_\mathcal{H})\biggr)\cdot (\Phi(\bx) -g) dF(\bx)\\
\nonumber& & - h \cdot \int \lim_{s\to 0}  \varphi(\|\Phi(\bx) -g - s\cdot h\|_\mathcal{H})\, dF(\bx)\\
\nonumber& = & - \int \biggl(\frac{\psi'(\|\Phi(\bx) - g\|_\mathcal{H})\cdot\|\Phi(\bx) - g\|_\mathcal{H} - \psi(\|\Phi(\bx) - g\|_\mathcal{H})}
{\|\Phi(\bx) - g\|_\mathcal{H}^2}\cdot\frac{\langle h, \Phi(\bx) - g \rangle_\mathcal{H}}{\|\Phi(\bx) - g\|_\mathcal{H}}\biggr)\\
\nonumber && \quad \quad \quad\cdot \bigl(\Phi(\bx) -g\bigr)\, dF(\bx)\\
& & - h \cdot  \int \varphi(\|\Phi(\bx) -g\|_\mathcal{H})\, dF(\bx)
\end{eqnarray}
where in the last equality, we use the fact
\begin{equation*}
\frac{\partial}{\partial s} \varphi(\|\Phi(\bx) - g - s\cdot h\|_\mathcal{H}) = \varphi'(\|\Phi(\bx) - g - s\cdot h\|_\mathcal{H})\cdot\frac{\langle \Phi(\bx) - g - s\cdot h, h\rangle_\mathcal{H}}{\|\Phi(\bx) - g - s\cdot h\|_\mathcal{H}}
\end{equation*}
and
\begin{equation*}
\varphi'(x) = \frac{d}{d x} \frac{\psi(x)}{x} = \frac{\psi'(x)x - \psi(x)}{x^2}.
\end{equation*}
The exchange of limit and integral is valid due to the dominated convergence theorem since under the assumption that $\varphi$ is bounded and
Lipschitz continuous with Lipschitz constant $L$,
\begin{eqnarray*}
\bigl|\varphi(\|\Phi(\bx) - g - s\cdot h\|) \bigr|< \infty, \quad \forall \bx
\end{eqnarray*}
and
\begin{eqnarray*}
\lefteqn{\biggl\|\frac{1}{s} \biggl(\varphi(\|\Phi(\bx) - g - s\cdot h\|_\mathcal{H}) -
\varphi(\|\Phi(\bx) -g\|_\mathcal{H})\biggr)\cdot \bigl(\Phi(\bx) -g\bigr)\biggr\|_\mathcal{H}}\\
& = & \frac{1}{s} \bigl |\varphi(\|\Phi(\bx) -g - s\cdot h\|_\mathcal{H}) -
\varphi(\|\Phi(\bx) -g\|_\mathcal{H})\bigr| \cdot \|\Phi(\bx) -g\|_\mathcal{H}\\
& \leq & \frac{1}{s} L \cdot \|s \cdot h\|_\mathcal{H} \cdot \bigl(\|\Phi(\bx)\|_\mathcal{H} + \|g\|_\mathcal{H}\bigl) \\
& \leq & L \cdot \| h\|_\mathcal{H} \cdot \bigl(\|\Phi(\bx)\|_\mathcal{H} + \|g\|_\mathcal{H}\bigl) < \infty, \quad \forall \bx.
\end{eqnarray*}

By combining (\ref{eqn:icderiv1}), (\ref{eqn:icderiv4}), (\ref{eqn:icderiv2}), and (\ref{eqn:icderiv3}), we have
\begin{eqnarray*}
\nonumber \lefteqn{\biggl(\int \varphi(\|\Phi(\bx) - f_{F}\|) dF \biggr)\cdot \dot{f}_{F} }\\
\nonumber & + & \int \biggl(\frac{\bigl\langle \dot{f}_{F}, \Phi(\bx) - f_{F} \bigr\rangle_\mathcal{H}}{\|\Phi(\bx) - f_{F}\|^3} \cdot
q(\|\Phi(\bx) - f_{F}\|)\cdot \bigl(\Phi(\bx) - f_{F}\bigr)\biggr) dF(\bx)\\
& = & (\Phi(\bx') - f_{F})\cdot \varphi(\|\Phi(\bx') - f_{F}\|)
\end{eqnarray*}
where $q(x) = x\psi'(x) - \psi(x)$.

\subsection{Proof of Theorem \ref{thm:influence_emp}}\label{ssec:proof_thm_influnece_emp}
With $F_n$ instead of $F$, (\ref{eqn:icfinal}) becomes
\begin{eqnarray}\label{eqn:m_sigma_dot}
\nonumber \lefteqn{\biggl(\frac{1}{n} \sum_{i = 1}^n \varphi(\|\Phi(\bX_i) - f_{F_n}\|)\biggr)\cdot \dot{f}_{F_n} }\\
\nonumber & + &\frac{1}{n} \sum_{i=1}^n \biggl(\frac{\bigl\langle \dot{f}_{F_n}, \Phi(\bX_i) - f_{F_n} \bigr\rangle_\mathcal{H}}{\|\Phi(\bX_i)
- f_{F_n}\|^3} \cdot q(\|\Phi(\bX_i) - f_{F_n}\|)\cdot \bigl(\Phi(\bX_i) - f_{F_n}\bigr)\biggr)\\
& = & (\Phi(\bx') - f_{F_n})\cdot \varphi(\|\Phi(\bx') - f_{F_n}\|).
\end{eqnarray}
Let  $r_i = \|\Phi(\bX_i) - f_{F_n}\|$, $r' = \|\Phi(\bx') - f_{F_n}\|$, $\gamma = \sum_{i=1}^n \varphi(r_i)$ and
\begin{equation*}
d_i = \bigl \langle \dot{f}_{F_n}, \Phi(\bX_i) - f_{F_n} \bigr \rangle_\mathcal{H} \cdot \frac{q(r_i)}{r_i^3}.
\end{equation*}
Then, (\ref{eqn:m_sigma_dot}) simplifies to
\begin{equation*}
\gamma \cdot \dot{f}_{F_n} + \sum_{i=1}^n d_i \cdot \bigl(\Phi(\bX_i) - f_{F_n}\bigr)  =  n\cdot(\Phi(\bx') - f_{F_n})\cdot\varphi(r')
\end{equation*}
Since $f_{F_n} = \sum_{i=1}^n w_i \Phi(\bX_i)$, we can see that $\dot{f}_{F_n}$ has a form of $\sum_{i=1}^n \alpha_i \Phi(\bX_i) +
\alpha' \Phi(\bx')$. By substituting this, we have
\begin{eqnarray*}
\lefteqn{\gamma \sum_{j=1}^n \alpha_j \Phi(\bX_j) + \gamma \cdot\alpha' \Phi(\bx') + \sum_{i=1}^n d_i \biggl(\Phi(\bX_i) -
\sum_{k=1}^n w_k \Phi(\bX_k)\biggr)}\\
& = & n\cdot\biggl(\Phi(\bx') - \sum_{k=1}^n w_k \Phi(\bX_k)\biggr)\cdot \varphi(r').
\end{eqnarray*}
Since $K'$ is positive definite, $\Phi(\bX_i)$'s and $\Phi(\bx')$ are linearly independent (see Lemma \ref{lem:lindep1}). Therefore, by comparing the coefficients of the $\Phi(\bX_j)$'s and $\Phi(\bx')$ in both sides, we have
\begin{gather}
\label{eqn:alpha_j} \gamma\cdot\alpha_j + d_j - w_j\cdot \biggl(\sum_{i=1}^n d_i\biggr)  = -w_j\frac{\psi(r')}{r'}\cdot n\\
\label{eqn:alpha_prime}\gamma\alpha' = n\cdot\varphi(r').
\end{gather}

From (\ref{eqn:alpha_prime}), $\alpha' = n \varphi(r')/\gamma$. Let $q_i = q(r_i)/r_i^3$ and $\Phi(\bX_i) - f_{F_n} = \sum_{k=1}^n w_{k, i}\Phi(\bX_k)$
where
\begin{equation*}
w_{k, i} =
\begin{cases}
- w_k&, \quad k \neq i\\
1 - w_k&, \quad k = i.
\end{cases}
\end{equation*}
Then,
\begin{align*}
d_i & = \frac{q(r_i)}{r_i^3}\biggl\langle \dot{f}_{F_n}, \Phi(\bX_i) - f_{F_n}\biggr\rangle_\mathcal{H}\\
& = q_i \biggl\langle \sum_{j=1}^n \alpha_j \Phi(\bX_j) + \alpha' \Phi(\bx'), \sum_{k=1}^n w_{k, i}\Phi(\bX_k) \biggr\rangle_\mathcal{H}\\
& = q_i \biggl( \sum_{j=1}^n \sum_{k=1}^n \alpha_j w_{k, i} k_\sigma(\bX_j, \bX_k) + \alpha' \sum_{k=1}^n  w_{k, i} k_\sigma(\bx', \bX_k) \biggr)\\
& = q_i (\be_i - \bw)^T K \balpha + q_i \alpha' \cdot(\be_i - \bw)^T \bk'\\
& = q_i (\be_i - \bw)^T \bigl(K \balpha + \alpha' \bk'\bigr)
\end{align*}
where $K : = (k_\sigma(\bX_i, \bX_j))_{i, j=1}^n$ is a kernel matrix, $\be_i$ denotes the $i$th standard basis vector, and
$\bk' = [k_\sigma(\bx', \bX_1, \dots, k_\sigma(\bx', \bX_n)]^T$. By letting $Q = diag([q_1, \dots, q_n])$,
\begin{equation*}
\bd = Q\cdot(I_n - \bones \bw^T) (K\balpha + \alpha' \cdot \bk').
\end{equation*}
Thus, (\ref{eqn:alpha_j}) can be expressed in matrix-vector form,
\begin{gather*}
\gamma \balpha + Q\cdot(I_n - \bones\cdot \bw^T) (K\balpha + \alpha' \cdot \bk') - \bw \cdot \bigl(\bones^TQ\cdot(I_n - \bones\cdot \bw^T)
(K\balpha + \alpha' \cdot \bk')\bigr)\\
 = - n\cdot\bw \varphi(r').
\end{gather*}
Thus, $\balpha$ can be found solving the following linear system of equations,
\begin{eqnarray*}
\lefteqn{\biggl\{\gamma I_n + (I_n - \bones\cdot\bw^T)^T Q\cdot(I_n - \bones \cdot\bw^T) \cdot K \biggr\} \balpha} \\
& = & -n \cdot \varphi(r') \bw  - \alpha' (I_n - \bones\cdot\bw^T)^T Q\cdot(I_n - \bones\cdot \bw^T)\bk'.
\end{eqnarray*}
Therefore,
\begin{align*}
IF(\bx, \bx'; \widehat{f}_{RKDE}, F_n) & = \biggl \langle \Phi(\bx), \dot{f}_{F_n} \biggr \rangle_\mathcal{H}\\
& = \biggl \langle \Phi(\bx), \sum_{i=1}^n \alpha_i \Phi(\bX_i) + \alpha' \Phi(\bx') \biggr \rangle_\mathcal{H}\\
& = \sum_{i=1}^n \alpha_i k_\sigma(\bx, \bX_i) + \alpha' k_\sigma(\bx, \bx').
\end{align*}

The condition $\lim_{s\to 0}f_{F_{n,s}} = f_{F_n}$ is implied by the strict convexity of $J$. Given $\bX_1, \dots, \bX_n$ and $\bx'$, define
$\mathcal{D}_{n+1}$ as in Lemma \ref{lemma:lemma_compact1}. From Theorem \ref{thm:representer}, $f_{F_n,s}$ and $f_{F_n}$ are in $\mathcal{D}_{n+1}$. With the definition in (\ref{eqn:extended_J}),
\begin{align*}
J_{F_{n,s}}(g) & = \int \rho(\|\Phi(\bx) - g\|_\mathcal{H}) \, dF_{n,s}(\bx)\\
& = \frac{(1-s)}{n} \sum_{i=1}^n \rho(\|\Phi(\bX_i) - g\|_\mathcal{H}) + s\cdot\rho(\|\Phi(\bx') - g\|_\mathcal{H}).
\end{align*}
Note that $J_{F_{n,s}}$ uniformly converges to $J$ on $\mathcal{D}_{n+1}$, i.e, $\sup_{g \in \mathcal{D}_{n+1}}|J_{F_{n,s}}(g) - J(g)| \to 0$ as $s\to 0$, since for any $g \in \mathcal{D}_{n+1}$
\begin{eqnarray*}
\lefteqn{\bigl|J_{F_{n,s}}(g) - J(g)\bigr|}\\
&= &\biggl|\frac{(1-s)}{n} \sum_{i=1}^n \rho(\|\Phi(\bX_i) - g\|_\mathcal{H}) + s\cdot\rho(\|\Phi(\bx') - g\|_\mathcal{H}) - \frac{1}{n} \sum_{i=1}^n \rho(\|\Phi(\bX_i) - g\|_\mathcal{H})\biggr|\\
&= &\frac{s}{n} \sum_{i=1}^n \rho(\|\Phi(\bX_i) - g\|_\mathcal{H}) + s\cdot\rho(\|\Phi(\bx') - g\|_\mathcal{H})\\
&\leq &\frac{s}{n} \sum_{i=1}^n \rho(2\tau) + s\cdot\rho(2\tau)\\
&= &2s\cdot\rho(2\tau)
\end{eqnarray*}
where in the inequality we use the fact that $\rho$ is nondecreasing and
\begin{align*}
\|\Phi(\bx) - g\|_\mathcal{H} &\leq \|\Phi(\bx)\| + \|g\|_\mathcal{H}\\
&\leq 2\tau.
\end{align*}
since $g \in \mathcal{D}_{n+1}$, and by the triangle inequality.

Now, let $\epsilon > 0$ and $B_\epsilon (f_{F_n}) \subset \mathcal{H}$ be the open ball centered at $f_{F_n}$ with radius $\epsilon$. Since $\mathcal{D}_{n+1}^\epsilon \triangleq \mathcal{D}_{n+1}\setminus B_\epsilon (f_{F_n})$ is also compact, $\inf_{g \in \mathcal{D}_{n+1}^\epsilon} J(g)$ is attained by some $g^* \in \mathcal{D}_{n+1}^\epsilon$ by the extreme value theorem \citep{adams08}. Since $f_{F_n}$ is unique,
$M_\epsilon =  J(g^*) - J(f_{F_n})> 0$. For sufficiently small $s$, $\sup_{g \in \mathcal{D}_{n+1}}|J_{F_{n,s}}(g) - J(g)| < M_\epsilon/2$ and thus
\begin{equation*}
J(g) - \frac{M_\epsilon}{2} < J_{F_{n,s}}(g) < J(g) + \frac{M_\epsilon}{2}, \quad \forall g \in \mathcal{D}_{n+1}.
\end{equation*}
Therefore,
\begin{align*}
\inf_{g \in \mathcal{D}_{n+1}^\epsilon} J_{F_{n,s}}(g) & > \inf_{g \in \mathcal{D}_{n+1}^\epsilon} J(g) -\frac{M_\epsilon}{2}\\
& = J(g^*) - \frac{M_\epsilon}{2}\\
& = J(f_{F_n}) + M_\epsilon - \frac{M_\epsilon}{2}\\
& = J(f_{F_n}) + \frac{M_\epsilon}{2}\\
& > J_{F_{n,s}}(f_{F_n})
\end{align*}
Since the minimum of $J_{F_{n,s}}$ is not attained on $\mathcal{D}_{n+1}^\epsilon$, $f_{F_{n, s}} \in B_\epsilon (f_{F_n})$. Since $\epsilon$ is arbitrary, $\lim_{s \to 0} f_{F_{n,s}} = f_{F_n}$.
\mbox{}\vspace*{1ex}
\mbox{}

\vskip 0.2in
\bibliographystyle{icml2011}
\bibliography{Paper_Kim_Clay}

\end{document}